\documentclass{book}
\usepackage{amssymb,amsmath,latexsym,epic,eepic,itemlist,kcp-new}
\usepackage[bindingoffset=0cm,a4paper,centering,totalheight=200mm,textwidth=120mm]{geometry}

\newcommand{\CA}{{\cal A}}
\newcommand{\comma}{,\ldots ,}
\newcommand{\Bf}{{\bf f}}
\newcommand{\half}{\mbox{$\frac{1}{2}$}}
\newcommand{\CM}{{\mathcal{M}}}
\newcommand{\Bh}{{\bf h}}
\newcommand{\BE}{{\bf E}}
\newcommand{\Bz}{{\bf z}}
\newcommand{\BBE}{\mbox{\(\mathbb E\)}}
\newcommand{\BBP}{\mbox{\(\mathbb P\)}}
\newcommand{\Be}{{\bf e}}
\newcommand{\CB}{\mathcal{B}}

\collection
\begin{document}
\title{The Equational Approach to CF2 Semantics}
\author{Dov M. Gabbay\\
Bar Ilan University, Israel;\\
King's College London, UK;\\
University of Luxembourg, Luxembourg\\
{\small Paper 459h:  DovPapers/459o/459-EACF2s.tex}\\
February 2012; Version dated 15 February 2012
}
\maketitle

\begin{abstract}
We introduce a family of new equational semantics for argumentation networks which can handle odd and even loops in a uniform manner. We offer one version of   equational semantics which is equivalent to CF2 semantics, and a better version which gives the same results as traditional Dung semantics for even loops but can still handle odd loops.
\end{abstract}

\section{Background and orientation}
Our starting point is the important papers of Baroni, Giacomin and Guida on odd loops and SCC recursiveness \cite{459-1,459-6}.  In their papers the authors offer the CF2 semantics in response to difficulties arising from the Dung semantics handling of odd and even loops. In our paper we outline our equational approach to argumentation networks and show how the CF2 semantics can be obtained from perturbations to the equations associated with the networks. This approach will offer additional methodological support for the CF2 semantics, while at the same time show the power of the equational approach. We offer our own loop-busting equational semantics LB, which includes CF2 as a special case.

The structure of this paper is as follows.  Section 2 reproduces the motivating discussion from \cite{459-1} for the CF2 semantics and points out its weaknesses.  Section 3 introduces the equational semantics.  Section 4 defines our loop busting semantics LB. Section 5 introduces our semantics LB2 and compares with CF2 on the technical level.  We conclude with a general discussion in Section 6. 

\section{CF2 semantics as introduced in the SCC paper \cite{459-1}}

Baroni {\em et al.} devote a long discussion about the inadequacy of the traditional semantics in handling odd and even loops. They say, and I quote:
\begin{quote}
``the length of the leftmost cycle should not affect the justification  states [of an argument]. More generally, it is counter-intuitive that different results in conceptually similar situations depend on the length of the cycle. Symmetry reasons suggest that all cycles should be treated equally and should yield the same results.''
\end{quote}

We now reproduce Figure 8 of \cite{459-1}, and discuss the problems associated with it.

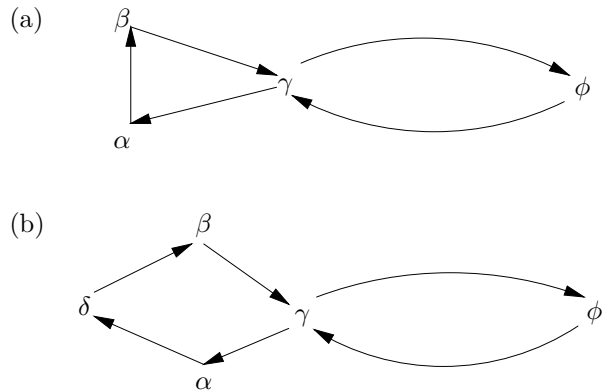
\begin{figure}[ht]
\centering
\setlength{\unitlength}{0.00083333in}
\begingroup\makeatletter\ifx\SetFigFont\undefined%
\gdef\SetFigFont#1#2#3#4#5{%
  \reset@font\fontsize{#1}{#2pt}%
  \fontfamily{#3}\fontseries{#4}\fontshape{#5}%
  \selectfont}%
\fi\endgroup%
{\renewcommand{\dashlinestretch}{30}
\begin{picture}(3593,2452)(0,-10)
\put(3503,1894){\makebox(0,0)[lb]{\smash{{\SetFigFont{10}{12.0}{\rmdefault}{\mddefault}{\updefault}$\phi$}}}}
\put(2737.000,-1541.000){\arc{4652.419}{4.3498}{5.0749}}
\blacken\path(3438.828,645.337)(3562.000,634.000)(3458.654,701.967)(3438.828,645.337)
\put(2565.808,359.154){\arc{3771.513}{4.3029}{5.2125}}
\blacken\path(3349.352,2041.278)(3470.000,2014.000)(3376.383,2094.844)(3349.352,2041.278)
\put(2637.419,3476.175){\arc{3628.683}{1.1126}{2.0743}}
\blacken\path(1882.601,1859.514)(1762.000,1887.000)(1855.478,1805.995)(1882.601,1859.514)
\path(765,1714)(765,2314)
\blacken\path(795.000,2194.000)(765.000,2314.000)(735.000,2194.000)(795.000,2194.000)
\path(765,2314)(1665,2014)
\blacken\path(1541.671,2023.487)(1665.000,2014.000)(1560.645,2080.408)(1541.671,2023.487)
\path(1665,1939)(765,1714)
\blacken\path(874.141,1772.209)(765.000,1714.000)(888.693,1714.000)(874.141,1772.209)
\path(1215,964)(1740,589)
\blacken\path(1624.915,634.337)(1740.000,589.000)(1659.789,683.161)(1624.915,634.337)
\path(1740,439)(1215,214)
\blacken\path(1313.480,288.845)(1215.000,214.000)(1337.115,233.696)(1313.480,288.845)
\path(1215,214)(540,514)
\blacken\path(661.842,492.678)(540.000,514.000)(637.473,437.849)(661.842,492.678)
\path(540,664)(1140,964)
\blacken\path(1046.085,883.502)(1140.000,964.000)(1019.252,937.167)(1046.085,883.502)
\put(15,2314){\makebox(0,0)[lb]{\smash{{\SetFigFont{10}{12.0}{\rmdefault}{\mddefault}{\updefault}(a)}}}}
\put(765,2314){\makebox(0,0)[rb]{\smash{{\SetFigFont{10}{12.0}{\rmdefault}{\mddefault}{\updefault}$\beta$}}}}
\put(765,1564){\makebox(0,0)[rb]{\smash{{\SetFigFont{10}{12.0}{\rmdefault}{\mddefault}{\updefault}$\alpha$}}}}
\put(15,1039){\makebox(0,0)[lb]{\smash{{\SetFigFont{10}{12.0}{\rmdefault}{\mddefault}{\updefault}(b)}}}}
\put(1215,1039){\makebox(0,0)[b]{\smash{{\SetFigFont{10}{12.0}{\rmdefault}{\mddefault}{\updefault}$\beta$}}}}
\put(1215,64){\makebox(0,0)[b]{\smash{{\SetFigFont{10}{12.0}{\rmdefault}{\mddefault}{\updefault}$\alpha$}}}}
\put(1672,1924){\makebox(0,0)[lb]{\smash{{\SetFigFont{10}{12.0}{\rmdefault}{\mddefault}{\updefault}$\gamma$}}}}
\put(520,522){\makebox(0,0)[rb]{\smash{{\SetFigFont{10}{12.0}{\rmdefault}{\mddefault}{\updefault}$\delta$}}}}
\put(1777,484){\makebox(0,0)[lb]{\smash{{\SetFigFont{10}{12.0}{\rmdefault}{\mddefault}{\updefault}$\gamma$}}}}
\put(3578,499){\makebox(0,0)[lb]{\smash{{\SetFigFont{10}{12.0}{\rmdefault}{\mddefault}{\updefault}$\phi$}}}}
\put(2693.188,1703.254){\arc{3001.697}{0.9834}{2.1300}}
\blacken\path(2016.099,397.602)(1897.000,431.000)(1986.371,345.485)(2016.099,397.602)
\end{picture}
}
\caption{Figure 8 of \cite{459-1}.  Problematic argumentation of frameworks}
\label{FF1}
\end{figure}

The only preferred extension for Figure \ref{FF1}(a) is $\{\phi,\alpha\}$, while for Figure \ref{FF1}(b) we have the extensions $\{\alpha,\beta,\phi\}$ and $\{\delta,\gamma\}$.  These two results are conceptually different, in (a) $\phi$ is not prevented from being justified while in (b) it is prevented.

A more striking problem is the one outlined in Figure 9 of Baroni {\em et al.} \cite{459-1}, here reproduced as Figure \ref{FF2}.

\begin{figure}[ht]
\centering
\setlength{\unitlength}{0.00083333in}
\begingroup\makeatletter\ifx\SetFigFont\undefined%
\gdef\SetFigFont#1#2#3#4#5{%
  \reset@font\fontsize{#1}{#2pt}%
  \fontfamily{#3}\fontseries{#4}\fontshape{#5}%
  \selectfont}%
\fi\endgroup%
{\renewcommand{\dashlinestretch}{30}
\begin{picture}(2982,952)(0,-10)
\put(2967,439){\makebox(0,0)[b]{\smash{{\SetFigFont{10}{12.0}{\rmdefault}{\mddefault}{\updefault}$\delta$}}}}
\put(427.714,-3578.857){\arc{8957.297}{4.6430}{5.0871}}
\blacken\path(1944.145,603.379)(2067.000,589.000)(1965.366,659.502)(1944.145,603.379)
\path(42,214)(42,814)
\blacken\path(72.000,694.000)(42.000,814.000)(12.000,694.000)(72.000,694.000)
\path(42,814)(942,514)
\blacken\path(818.671,523.487)(942.000,514.000)(837.645,580.408)(818.671,523.487)
\path(942,439)(42,214)
\blacken\path(151.141,272.209)(42.000,214.000)(165.693,214.000)(151.141,272.209)
\path(1242,514)(1842,514)
\blacken\path(1722.000,484.000)(1842.000,514.000)(1722.000,544.000)(1722.000,484.000)
\path(2292,514)(2892,514)
\blacken\path(2772.000,484.000)(2892.000,514.000)(2772.000,544.000)(2772.000,484.000)
\put(42,814){\makebox(0,0)[rb]{\smash{{\SetFigFont{10}{12.0}{\rmdefault}{\mddefault}{\updefault}$\beta$}}}}
\put(42,64){\makebox(0,0)[rb]{\smash{{\SetFigFont{10}{12.0}{\rmdefault}{\mddefault}{\updefault}$\alpha$}}}}
\put(1017,439){\makebox(0,0)[b]{\smash{{\SetFigFont{10}{12.0}{\rmdefault}{\mddefault}{\updefault}$\phi$}}}}
\put(2067,439){\makebox(0,0)[b]{\smash{{\SetFigFont{10}{12.0}{\rmdefault}{\mddefault}{\updefault}$\gamma$}}}}
\put(688.323,3750.029){\arc{7311.892}{1.1841}{1.7277}}
\blacken\path(1965.914,292.714)(2067.000,364.000)(1944.260,348.671)(1965.914,292.714)
\end{picture}
}
\caption{Figure 9 of \cite{459-1}: Floating defeat and floating acceptance}\label{FF2}
\end{figure}
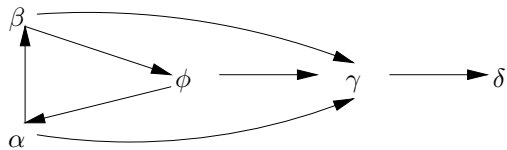

The only extension in traditional Dung semantics is all undecided.  Common sense, however, expects $\gamma$ to be out and $\delta$ to be in.  In \cite{459-5}, this is characterised as ``one of the main unsolved problems in argumentation-based semantics''.  

The CF2 semantics of \cite{459-1} treats the loops of Figures \ref{FF1}(a) and \ref{FF1}(b) and \ref{FF2} all in the same way, by taking as CF2 extensions maximal conflict-free sets.  We therefore get for Figure \ref{FF1}(a) the CF2 extensions
\[
\{\alpha,\phi\},\{\beta,\phi\} \mbox{ and } \{\gamma\}
\]
and for Figure \ref{FF1}(b) we get
\[
\{\alpha,\beta,\phi\}, \{\delta,\gamma\}\mbox{ and }\{\delta,\phi\}
\]
and for Figure \ref{FF2} we get the extensions 
\[
\{\beta,\delta\},\{\alpha,\delta\}\mbox{ and } \{\phi,\delta\}.
\]

Let us put forward a figure of our own, Figure \ref{FF3}.  This is a 9 point cycle.  The CF2 semantics will take all maximal conflict free subsets as extensions, including among them $\{a_3, a_6, a_9\}$ and its cyclic translations as well as $\{a_1,a_3,a_5,a_7\}$ and its cyclic translations (e.g.\ $\{a_2,a_4,a_6,a_8\}$, etc.).

\begin{figure}[ht]
\centering
\setlength{\unitlength}{0.00083333in}
\begingroup\makeatletter\ifx\SetFigFont\undefined%
\gdef\SetFigFont#1#2#3#4#5{%
  \reset@font\fontsize{#1}{#2pt}%
  \fontfamily{#3}\fontseries{#4}\fontshape{#5}%
  \selectfont}%
\fi\endgroup%
{\renewcommand{\dashlinestretch}{30}
\begin{picture}(3255,2476)(0,-10)
\put(540,2089){\makebox(0,0)[b]{\smash{{\SetFigFont{10}{12.0}{\rmdefault}{\mddefault}{\updefault}$a_9$}}}}
\put(2250.934,1166.588){\arc{2062.771}{5.2375}{5.9957}}
\blacken\path(3171.371,1561.908)(3240.000,1459.000)(3227.873,1582.097)(3171.371,1561.908)
\put(2167.618,1338.360){\arc{2149.915}{0.0692}{0.9187}}
\blacken\path(2891.769,584.743)(2820.000,484.000)(2930.734,539.118)(2891.769,584.743)
\put(1652.081,2299.432){\arc{4330.980}{1.1624}{1.5024}}
\blacken\path(1916.536,180.466)(1800.000,139.000)(1922.339,120.748)(1916.536,180.466)
\put(2021.343,3682.266){\arc{7194.292}{1.7565}{1.9882}}
\blacken\path(685.179,374.708)(563.000,394.000)(661.727,319.481)(685.179,374.708)
\put(1340.362,1500.135){\arc{2739.211}{2.3770}{2.8672}}
\blacken\path(87.907,1024.328)(22.000,1129.000)(30.896,1005.627)(87.907,1024.328)
\put(916.548,1258.735){\arc{1781.862}{3.2408}{4.1704}}
\blacken\path(375.856,1928.642)(457.000,2022.000)(341.470,1977.810)(375.856,1928.642)
\put(1730.297,826.952){\arc{3143.841}{4.7612}{5.2351}}
\blacken\path(2394.819,2218.266)(2515.000,2189.000)(2422.730,2271.378)(2394.819,2218.266)
\put(1590,2314){\makebox(0,0)[b]{\smash{{\SetFigFont{10}{12.0}{\rmdefault}{\mddefault}{\updefault}$a_1$}}}}
\put(2640,2089){\makebox(0,0)[b]{\smash{{\SetFigFont{10}{12.0}{\rmdefault}{\mddefault}{\updefault}$a_2$}}}}
\put(3240,1339){\makebox(0,0)[b]{\smash{{\SetFigFont{10}{12.0}{\rmdefault}{\mddefault}{\updefault}$a_3$}}}}
\put(2640,364){\makebox(0,0)[b]{\smash{{\SetFigFont{10}{12.0}{\rmdefault}{\mddefault}{\updefault}$a_4$}}}}
\put(1590,64){\makebox(0,0)[b]{\smash{{\SetFigFont{10}{12.0}{\rmdefault}{\mddefault}{\updefault}$a_5$}}}}
\put(465,439){\makebox(0,0)[b]{\smash{{\SetFigFont{10}{12.0}{\rmdefault}{\mddefault}{\updefault}$a_6$}}}}
\put(15,1189){\makebox(0,0)[b]{\smash{{\SetFigFont{10}{12.0}{\rmdefault}{\mddefault}{\updefault}$a_7$}}}}
\put(1512.612,580.767){\arc{3609.252}{4.2391}{4.6511}}
\blacken\path(1285.407,2340.695)(1402.000,2382.000)(1279.686,2400.422)(1285.407,2340.695)
\end{picture}
}
\caption{}\label{FF3}
\end{figure}

We shall see later that some of our loop busting semantics LB yield only $\{a_1,a_3,a_5,a_7\}$ and its cyclic translations and not $\{a_3,a_6,a_9\}$, but other LB semantics does yield it.

We agree with \cite{459-1} on the need for a new approach but we feel that the CF2 semantics offered as a solution requires further independent methodological justification. The notion of conflict freeness is a neutral notion and does not use the central notion of ``attack'' of the Dung semantics.  When we get a loop like $\{\alpha,\beta,\gamma\}$, in a real life application as in Figure \ref{FF1}(a), there are good reasons for the loop in the context of the application area where it arises, and we want a decisive solution to the loop in terms of \{in, out\}, which makes sense in the application area. We do not want just a technical, non-decisive choice of maximal conflict free sets, a sort of compromise which involves no real decision making.  Imagine we have a loop with $\{\alpha, \beta, \gamma\}$, and we go to a judge and we expect some effective decision making.  We hope for something like ``I think $\gamma$ is not serious''.  

Taking the maximal conflict free sets in this case, namely $\{\alpha\},\{\beta\}$ and $\{\gamma\}$ means nothing.  We would perceive that the judge is not doing his job properly and that he is just offering us options which are obvious and non-controversial, given the geometry of the loop!  See \cite{459-3} for extensive examples of resolving loops in a practical realistic way.  

Another problem, in our opinion, with the CF2 semantics is that it is an overkill as far as loop-breaking is concerned.  If we look at Figures \ref{FF1}(a) and \ref{FF1}(b) and replace them by Figure \ref{FF4a} and \ref{FF4b} our loop-breaking needs are the same according to \cite{459-1}, but in Figure \ref{FF4a}, we do not need the extensions $\{a_3,a_6,a_9,b\}$ from the loop-breaking point of view. In our LB semantics we do not mind if there will be less extensions than CF2, in the odd cycle case of Figure \ref{FF4a} but insist that there will be the same extensions as the traditional Dung semantics in the even cycle of Figure \ref{FF4b}. CF2 gives more extensions then the traditional Dung extensions for Figure \ref{FF4b}.

\begin{figure}[ht]
\centering
\setlength{\unitlength}{0.00083333in}
\begingroup\makeatletter\ifx\SetFigFont\undefined%
\gdef\SetFigFont#1#2#3#4#5{%
  \reset@font\fontsize{#1}{#2pt}%
  \fontfamily{#3}\fontseries{#4}\fontshape{#5}%
  \selectfont}%
\fi\endgroup%
{\renewcommand{\dashlinestretch}{30}
\begin{picture}(3255,3307)(0,-10)
\put(1597,55){\makebox(0,0)[b]{\smash{{\SetFigFont{10}{12.0}{\rmdefault}{\mddefault}{\updefault}$b$}}}}
\put(1594.344,1053.135){\arc{4370.474}{4.8099}{5.1958}}
\blacken\path(2488.996,3013.654)(2610.000,2988.000)(2515.305,3067.579)(2488.996,3013.654)
\put(2250.934,1997.588){\arc{2062.771}{5.2375}{5.9957}}
\blacken\path(3171.371,2392.908)(3240.000,2290.000)(3227.873,2413.097)(3171.371,2392.908)
\put(2167.618,2169.360){\arc{2149.915}{0.0692}{0.9187}}
\blacken\path(2891.769,1415.743)(2820.000,1315.000)(2930.734,1370.118)(2891.769,1415.743)
\put(1652.081,3130.432){\arc{4330.980}{1.1624}{1.5024}}
\blacken\path(1916.536,1011.466)(1800.000,970.000)(1922.339,951.748)(1916.536,1011.466)
\put(2021.343,4513.266){\arc{7194.292}{1.7565}{1.9882}}
\blacken\path(685.179,1205.708)(563.000,1225.000)(661.727,1150.481)(685.179,1205.708)
\put(1340.362,2331.135){\arc{2739.211}{2.3770}{2.8672}}
\blacken\path(87.907,1855.328)(22.000,1960.000)(30.896,1836.627)(87.907,1855.328)
\put(916.548,2089.735){\arc{1781.862}{3.2408}{4.1704}}
\blacken\path(375.856,2759.642)(457.000,2853.000)(341.470,2808.810)(375.856,2759.642)
\path(1590,790)(1590,280)
\blacken\path(1560.000,400.000)(1590.000,280.000)(1620.000,400.000)(1560.000,400.000)
\put(1590,3145){\makebox(0,0)[b]{\smash{{\SetFigFont{10}{12.0}{\rmdefault}{\mddefault}{\updefault}$a_1$}}}}
\put(2640,2920){\makebox(0,0)[b]{\smash{{\SetFigFont{10}{12.0}{\rmdefault}{\mddefault}{\updefault}$a_2$}}}}
\put(3240,2170){\makebox(0,0)[b]{\smash{{\SetFigFont{10}{12.0}{\rmdefault}{\mddefault}{\updefault}$a_3$}}}}
\put(2640,1195){\makebox(0,0)[b]{\smash{{\SetFigFont{10}{12.0}{\rmdefault}{\mddefault}{\updefault}$a_4$}}}}
\put(1590,895){\makebox(0,0)[b]{\smash{{\SetFigFont{10}{12.0}{\rmdefault}{\mddefault}{\updefault}$a_5$}}}}
\put(465,1270){\makebox(0,0)[b]{\smash{{\SetFigFont{10}{12.0}{\rmdefault}{\mddefault}{\updefault}$a_6$}}}}
\put(15,2020){\makebox(0,0)[b]{\smash{{\SetFigFont{10}{12.0}{\rmdefault}{\mddefault}{\updefault}$a_7$}}}}
\put(540,2920){\makebox(0,0)[b]{\smash{{\SetFigFont{10}{12.0}{\rmdefault}{\mddefault}{\updefault}$a_9$}}}}
\put(1512.612,1411.767){\arc{3609.252}{4.2391}{4.6511}}
\blacken\path(1285.407,3171.695)(1402.000,3213.000)(1279.686,3231.422)(1285.407,3171.695)
\end{picture}
}

\caption{}\label{FF4a}
\end{figure}

\begin{figure}[ht]
\centering
\setlength{\unitlength}{0.00083333in}
\begingroup\makeatletter\ifx\SetFigFont\undefined%
\gdef\SetFigFont#1#2#3#4#5{%
  \reset@font\fontsize{#1}{#2pt}%
  \fontfamily{#3}\fontseries{#4}\fontshape{#5}%
  \selectfont}%
\fi\endgroup%
{\renewcommand{\dashlinestretch}{30}
\begin{picture}(3271,3307)(0,-10)
\put(794,995){\makebox(0,0)[b]{\smash{{\SetFigFont{10}{12.0}{\rmdefault}{\mddefault}{\updefault}$a_6$}}}}
\put(1766.112,2429.436){\arc{3162.460}{0.4809}{0.8514}}
\blacken\path(2874.759,1344.131)(2808.000,1240.000)(2915.908,1300.464)(2874.759,1344.131)
\put(1800.142,1281.010){\arc{3894.796}{4.7329}{5.0553}}
\blacken\path(2331.584,3123.276)(2455.000,3115.000)(2349.998,3180.381)(2331.584,3123.276)
\put(2419.453,2241.317){\arc{1605.550}{5.0337}{5.8018}}
\blacken\path(3041.786,2698.679)(3131.000,2613.000)(3092.601,2730.582)(3041.786,2698.679)
\put(1735.990,2013.000){\arc{2993.656}{6.0043}{6.3634}}
\blacken\path(3202.943,2014.129)(3228.000,1893.000)(3262.893,2011.670)(3202.943,2014.129)
\put(1603.652,4963.861){\arc{8059.271}{1.6243}{1.7217}}
\blacken\path(1120.954,993.503)(998.000,980.000)(1112.843,934.054)(1120.954,993.503)
\put(1755.500,3382.500){\arc{4978.077}{2.0007}{2.1791}}
\blacken\path(449.698,1298.994)(333.000,1340.000)(416.672,1248.901)(449.698,1298.994)
\put(1393.833,1923.480){\arc{2798.009}{3.2789}{3.5980}}
\blacken\path(117.288,2418.053)(138.000,2540.000)(62.338,2442.147)(117.288,2418.053)
\put(1388.638,1830.993){\arc{2765.569}{4.3843}{4.7459}}
\blacken\path(1315.325,3181.731)(1435.000,3213.000)(1314.689,3241.727)(1315.325,3181.731)
\put(1077.792,2126.031){\arc{2153.799}{2.5269}{2.9915}}
\blacken\path(66.955,1853.695)(13.000,1965.000)(8.229,1841.399)(66.955,1853.695)
\put(2128.284,419.356){\arc{5962.942}{4.0257}{4.1812}}
\blacken\path(531.486,2901.595)(618.000,2990.000)(500.062,2952.708)(531.486,2901.595)
\path(1623,790)(1623,280)
\blacken\path(1593.000,400.000)(1623.000,280.000)(1653.000,400.000)(1593.000,400.000)
\put(1623,3145){\makebox(0,0)[b]{\smash{{\SetFigFont{10}{12.0}{\rmdefault}{\mddefault}{\updefault}$a_1$}}}}
\put(1623,895){\makebox(0,0)[b]{\smash{{\SetFigFont{10}{12.0}{\rmdefault}{\mddefault}{\updefault}$a_6$}}}}
\put(1630,55){\makebox(0,0)[b]{\smash{{\SetFigFont{10}{12.0}{\rmdefault}{\mddefault}{\updefault}$b$}}}}
\put(2477,3010){\makebox(0,0)[b]{\smash{{\SetFigFont{10}{12.0}{\rmdefault}{\mddefault}{\updefault}$a_2$}}}}
\put(3168,2500){\makebox(0,0)[b]{\smash{{\SetFigFont{10}{12.0}{\rmdefault}{\mddefault}{\updefault}$a_3$}}}}
\put(2636,1112){\makebox(0,0)[b]{\smash{{\SetFigFont{10}{12.0}{\rmdefault}{\mddefault}{\updefault}$a_5$}}}}
\put(3220,1765){\makebox(0,0)[b]{\smash{{\SetFigFont{10}{12.0}{\rmdefault}{\mddefault}{\updefault}$a_4$}}}}
\put(283,1385){\makebox(0,0)[b]{\smash{{\SetFigFont{10}{12.0}{\rmdefault}{\mddefault}{\updefault}$a_7$}}}}
\put(33,2018){\makebox(0,0)[b]{\smash{{\SetFigFont{10}{12.0}{\rmdefault}{\mddefault}{\updefault}$a_8$}}}}
\put(203,2590){\makebox(0,0)[b]{\smash{{\SetFigFont{10}{12.0}{\rmdefault}{\mddefault}{\updefault}$a_9$}}}}
\put(723,3010){\makebox(0,0)[b]{\smash{{\SetFigFont{10}{12.0}{\rmdefault}{\mddefault}{\updefault}$a_{10}$}}}}
\put(1518.475,4296.672){\arc{6683.015}{1.2774}{1.4765}}
\blacken\path(1948.863,1013.309)(1833.000,970.000)(1955.613,953.690)(1948.863,1013.309)
\end{picture}
}
\caption{}\label{FF4b}
\end{figure}

We should realise that within the context of argumentation theory alone, the maximal conflict free CF2 solution seems somewhat arbitrary, a device which is just technically successful. 

It is also the only device available for the loop breaking in this context.

The next sections will discuss the equational approach of \cite{459-4}, and introduce the new LB semantics.

\section{The equational approach}

Let $\CA =(S,R)$ be an argumentation frame $S\neq \varnothing$ is the set of arguments and $R\subseteq S\times S$ is the attack relation.  The equational approach views $(S,R)$ as a bearer of equations with the elements of $S$ as the variables ranging over $[0,1]$ and with $R$ as the generator of equations. Let $x \in S$ and let $y_1\comma y_k$ be all of its attackers. We write two types of equations $Eq_{\max}(\CA)$ and $Eq_{\rm inverse}(\CA)$.\footnote{In \cite{459-4} there are more $Eq$ options.}

For $Eq_{\max}$ we write
\begin{itemize}
\item $x=1-\max(y_1\comma y_k)$
\item $x=1$ if it has no attackers.
\end{itemize}

For $Eq_{\rm inverse}$ we write 
\begin{itemize}
\item $x =\prod^k_{i=1} (1-y_i)$
\item $x=1$, if it has no attackers.
\end{itemize}

We seek solutions \Bf\ for the above equations. In \cite{459-4} we prove the following:

\begin{theorem}\label{FT1}
\begin{enumerate}
\item There is always at least one solution in $[0,1]$  to any system of continuous equations $Eq(\CA)$.
\item If we use $Eq_{\max} (\CA)$ then the solutions \Bf\ correspond exactly to the Dung extensions of \CA.  Namely
\begin{itemize}
\item $\Bf(x) =1$ corresponds to $x=$ in
\item $\Bf(x)=0$ corresponds to $x=$ out
\item $0 < \Bf(x) < 1$ corresponds to $x=$ undecided.

The actual value in $[0,1]$ reflects the degree of odd looping involving $x$.
\end{itemize}
\item  If we use $Eq_{\rm inverse}$, we give more sensitivity to loops. For example the more undecided elements $y$ attack  $x$, the closer to 0 (out)  its value gets.
\end{enumerate}
\end{theorem}

In the context of equations, a very natural step to take is to look at {\em Perturbations}.  If the equations describe a physical or economic system in equilibrium, we want to change the solution a bit (perturb the variables) and see how it affects the system. For example, when we go to the bank to negotiate a mortgage, we start with the amount we want to borrow and indicate for how many years we want the loan and then solve equations that tell us what the monthly payment is going to be. We then might change the amount or the number of years or even negotiate the interest rate if we find the monthly payments too high.

In the equational system arising from an argumentation network we can try and fix the value of some arguments and see what happens. In the equational context, this move is quite natural. We shall see later, that fixing some values to 0 in the equations of $Eq(\CA)$, amounts to adopting the CF2 semantics, when done in a certain way.  When done in other ways it gives the new loop-busting semantics LB.

\begin{example}\label{FE1}
Consider Figure \ref{FF2}.  The equations for this figure are (we use $Eq_{\rm inverse}$)
\begin{enumerate}
\item $\alpha =1-\phi$
\item $\beta =1-\alpha$
\item $\phi =1-\beta$
\item $\gamma =(1-\alpha)(1-\beta)(1-\phi)$
\item $\delta =1-\gamma$
\end{enumerate}
The solution here is 
\[\begin{array}{l}

\alpha=\beta=\phi=\half\\
\gamma=\frac{1}{8}\\
\delta=\frac{7}{8}
\end{array}
\]
Let us perturb the equation by adding an external force which makes a node equal zero. The best analogy I can think of is in electrical networks where you make the voltage of a node 0 by connecting it to earth.

Let $Z(x)$ be the ``earth'' connection for node $x$. We now do several perturbations as examples
\begin{enumerate}
\renewcommand{\labelenumi}{(\alph{enumi}).}
\item Let's choose to make $\phi=0$.  

We replace equation 3 by 
\item [$3^*a$.] $\phi =(1-\beta) Z(\phi)$
\item [$3^*b$.] $Z(\phi)=0$.\footnote{We use $Z$ and write $3^*a$ and $3^*b$, rather than just writing $3^*$ $a=0$ because of algebraic considerations. The current equations can be manipulated algebraically to to prove  $a=b=c$. By adding a fourth variable $Z(\phi)$ we prevent that.  }

The equations now solve to 
\[\begin{array}{l}
\phi=0, \alpha =1,\beta=0\\
\gamma=0\\
\delta=1.
\end{array}
\]
This gives us the extension $\{\alpha,\delta\}$
\item  If we try to make $\alpha=0$, we replace equation (1) by 
\item [$1^*a$.] $\alpha =(1-\phi)Z(\alpha)$
\item [$1^*b$.] $Z(\alpha) =0$

We solve the equations and get
\[
\begin{array}{l}
\alpha=0,\beta=1,\phi=0\\
\gamma=0\\
\delta=1
\end{array}
\]
This corresponds to the extension $\{\beta,\delta\}$.
\item Now let us make $\beta=0$. We replace equation (1) by 
\item [$2^*a$.] $\beta =(1-\alpha)Z(\beta)$
\item [$2^*b$.] $Z(\beta)=0$

Solving the new equations gives us 
\[\begin{array}{l}
\beta=0, \phi=1,\alpha=0\\
\gamma=0\\
\delta =1
\end{array}\]
This gives us the extension $\{\phi,\delta\}$.
\end{enumerate}

If we compare these extensions with the CF2 extensions, we see that they are the same.
\end{example}

\begin{example}\label{FE2}
Let us see what happens with Figure \ref{FF1}(b). Here we have a well behaved even loop. Let us write the equations 
\begin{enumerate}
\item $\alpha =1-\gamma$
\item $\delta =1-\alpha$
\item $\beta =1-\delta$
\item $\gamma =1-\beta$
\item $\phi =1-\gamma$
\end{enumerate}
Let us do some perturbations:
\begin{enumerate}
\renewcommand{\labelenumi}{(\alph{enumi})}
\item Let us make $\gamma =0$. We change equation 4 to 
\item [$4^*a$.] $\gamma =(1-\beta) Z(\gamma)$
\item [$4^*b$.] $Z(\gamma)=0$

We solve the new equations and get 
\[
\gamma =0, \alpha =1, \delta =0, \beta=1, \phi =1.
\]
The extension is $\{\alpha,\beta,\phi\}$.
\item Let us try $\alpha =0$. we replace equation 1 by 
\item [$1^*a$.]   $\alpha =(1-\gamma) Z(\alpha)$
\item [$1^*b$.] $Z (\alpha)=0$

We solve the new equations and get 
\[
\alpha=0, \delta =1, \beta=0,\gamma=1,\phi=0
\]
The extension we get is $\{\delta,\gamma\}$.

\item Let us make $\delta =0$. We replace equation 2 by 
\item [$2^*a$.] $\delta =(1-\alpha)Z(\delta)$
\item [$2^*b$.] $Z (\delta)=0$

The solution is 
\[
\delta =0, \alpha=1, \beta =1, \gamma=0 \mbox{ and } \phi=1
\]
This gives the extension 
\[
\{\alpha,\beta,\phi\}
\]
\item Let us make $\beta=0$. the new equations for $\beta$ are 
\item [$3^*a$.] $\beta =(1-\delta)Z(\beta)$
\item [$3^*b$.] $Z(\beta)=0$

We solve the new set of equations and get 
\[
\beta =0, \gamma=1,\phi=0, \alpha=0,\delta =1.
\]
The extension is $\{\gamma,\delta\}$.
\item Let us make $\phi =0$.  We change equation 5 to 
\item [$5^*a$.] $\phi =(1-\gamma)Z(\phi)$
\item [$5^*b$.] $Z(\phi)=0$

We solve the new equations. 

From (3) and (4) we get 
\item [5.] $\delta =\gamma$

From (1) and (2) we get 
\item [7.] $\alpha =\beta$.

Let $\alpha =\beta =x$.  Then $\gamma =\delta =1-x$.

If we want $\{0,1\}$ extensions, i.e.\ $x\in \{0,1\}$, then we get the extensions\\ $\{\alpha,\beta\}$, case $\{x=1, \phi=0\}$\\
$\{\gamma,\delta\}$, case $\{x=0, \phi=0\}$.

\item Let us make $\alpha =\gamma=0$.  The new equations are
\item [$1^*a$.] $\alpha =(1-\gamma)Z(\alpha)$
\item [$1^*b$.] $Z(\alpha)=0$
\item [2.] $\delta =1-\alpha$
\item [3.] $\beta =1-\delta$
\item [$4^*a$.] $\gamma =(1-\beta)Z(\gamma)$
\item [$4^*b$.] $Z(\gamma)=0$
\item [5.] $\phi =1-\gamma$.

The solution is 
\[\begin{array}{l}
\alpha =\gamma =0\\
\delta=1\\
\beta=0\\
\phi=1
\end{array}\]

The extension we get is $\{\delta,\phi\}$.

\item Let us summarise in Table \ref{A}.

\begin{table}[ht]
\begin{center}
\begin{tabular}{l|p{4cm}|p{4cm}}
Case&Set $B$ of points made 0 & Corresponding extensions\\
\hline
(a) &$B_a=\{\gamma\}$&$\{\alpha,\beta,\phi\}$\\
(b)&$B_b=\{\alpha\}$&$\{\delta,\gamma\}$\\
(c)& $B_c=\{\delta\} $&$\{\alpha,\beta,\phi\}$\\
(d)&$B_d=\{\beta\} $&$\{\gamma,\delta\} $\\
(e)&$B_e=\{\phi\} $&$\{\alpha,\beta\}, \{\gamma,\delta\} $\\
(f)&$B_f=\{\alpha,\gamma\} $&$\{\delta,\phi\} $\\
\hline
\end{tabular}
\end{center}
\caption{}\label{A}
\end{table}
\end{enumerate}
\end{example}

\section{The equational loop-busting semantics LB for complete loops}

We now introduce our loop busting semantics, the LB semantics for complete loops. We need  a series of concepts leading up to it.

\begin{definition}[Loops]\label{FD1}
Let $\CA =(S,R)$ be an argumentation network.
\begin{enumerate}
\item A subset $E =\{x_1\comma x_n\} \subseteq S$ is a loop cycle, (or a loop set, or a loop) if we have 
\[x_1Rx_2,x_2Rx_3\comma x_{n-1}Rx_n,x_nRx_1
\]
$(S, R)$ is said to be a complete loop if every element of $S$ is an element of some loop cycle.\footnote{Comparing with the terminology of  \cite{459-1}, a complete loop is a union of disjoint strongly connected sets.}

\item A set $B\subseteq S$ is a loop-buster if for every loop set $E$ we have $E \cap S\neq \varnothing$
\item Let $B\subseteq S$ be a loop-buster and let $\CM$ be a meta-predicate describing properties of $B$. We can talk about the semantics LB$\CM$, where, (when we define it later), we use only loop-busters $B$ such that $\CM(B)$ holds.  
Criteria for adequacy for LB$\CM$ are
\begin{enumerate}
\item It busts all odd numbered loops
\item It busts all even numbered loops and yields all allowable Dung extensions for such loops.\end{enumerate}
\item Our first two proposals for conditions $\CM$ on loop-busters is minimality. The idea is the smaller $B$ is, the more options we have.

Therefore, we define: A loop-buster set $B$ is minimal absolute if there is no loop-buster set $B'$ with a smaller number of elements (we do not require $B'\subseteq B$!).
\item A loop-buster set $B$ is minimal relative if there does not exist a $B' \varsubsetneqq B$ which is a loop-buster set.
\end{enumerate}
\end{definition}

\begin{example}[Loop-buster 1]\label{FE3}
Consider Figures \ref{FF5} and \ref{FF6}.

\begin{figure}
\centering
\setlength{\unitlength}{0.00083333in}
\begingroup\makeatletter\ifx\SetFigFont\undefined%
\gdef\SetFigFont#1#2#3#4#5{%
  \reset@font\fontsize{#1}{#2pt}%
  \fontfamily{#3}\fontseries{#4}\fontshape{#5}%
  \selectfont}%
\fi\endgroup%
{\renewcommand{\dashlinestretch}{30}
\begin{picture}(2015,971)(0,-10)
\put(322,806){\makebox(0,0)[b]{\smash{{\SetFigFont{10}{12.0}{\rmdefault}{\mddefault}{\updefault}$a$}}}}
\path(464,889)(466,889)(470,889)
	(477,889)(487,889)(501,889)
	(516,888)(533,888)(552,886)
	(572,884)(592,881)(615,877)
	(638,872)(664,864)(691,855)
	(719,844)(750,830)(776,816)
	(796,804)(811,794)(821,787)
	(828,781)(833,776)(837,771)
	(843,765)(851,756)(862,745)
	(877,731)(895,714)(914,694)
	(933,672)(947,652)(958,635)
	(964,620)(969,607)(972,596)
	(973,586)(974,566)
\blacken\path(938.045,684.352)(974.000,566.000)(997.970,687.348)(938.045,684.352)
\path(982,379)(981,377)(980,372)
	(978,364)(974,352)(970,338)
	(964,323)(957,306)(949,288)
	(939,269)(926,249)(911,228)
	(892,206)(869,184)(845,163)
	(822,146)(804,133)(790,124)
	(781,118)(775,113)(770,111)
	(764,108)(757,105)(746,101)
	(730,95)(707,88)(678,79)
	(644,71)(613,65)(584,62)
	(556,60)(531,59)(508,60)
	(486,61)(465,63)(446,66)
	(428,68)(412,71)(398,74)(374,79)
\blacken\path(497.596,83.895)(374.000,79.000)(485.359,25.156)(497.596,83.895)
\path(1035,424)(1035,422)(1037,417)
	(1038,410)(1041,399)(1045,386)
	(1050,371)(1056,355)(1064,338)
	(1073,320)(1085,301)(1100,281)
	(1118,259)(1140,236)(1161,217)
	(1181,201)(1198,187)(1211,177)
	(1221,169)(1229,164)(1234,160)
	(1239,157)(1244,154)(1250,151)
	(1258,147)(1270,141)(1287,135)
	(1308,127)(1334,118)(1364,109)
	(1394,102)(1423,97)(1450,94)
	(1475,93)(1497,93)(1519,93)
	(1539,95)(1558,97)(1575,100)
	(1591,102)(1604,104)(1627,109)
\blacken\path(1516.112,54.193)(1627.000,109.000)(1503.366,112.824)(1516.112,54.193)
\path(1874,244)(1876,246)(1879,250)
	(1885,256)(1893,265)(1903,276)
	(1914,289)(1925,304)(1936,320)
	(1947,337)(1958,357)(1969,379)
	(1978,404)(1987,431)(1993,458)
	(1998,482)(2000,501)(2002,514)
	(2003,523)(2003,529)(2003,533)
	(2003,537)(2002,543)(2000,552)
	(1998,565)(1993,584)(1987,607)
	(1979,634)(1968,661)(1957,685)
	(1944,706)(1932,725)(1920,742)
	(1907,757)(1895,771)(1883,784)
	(1873,794)(1852,814)
\blacken\path(1959.586,752.966)(1852.000,814.000)(1918.207,709.517)(1959.586,752.966)
\path(1612,926)(1609,927)(1604,928)
	(1594,930)(1581,933)(1564,936)
	(1545,939)(1525,942)(1502,944)
	(1478,944)(1452,943)(1423,940)
	(1391,935)(1357,926)(1328,916)
	(1302,906)(1281,897)(1264,890)
	(1252,884)(1244,879)(1237,875)
	(1233,872)(1228,868)(1222,864)
	(1215,858)(1205,849)(1191,838)
	(1174,823)(1153,805)(1132,784)
	(1109,759)(1090,735)(1075,713)
	(1062,692)(1052,673)(1044,655)
	(1037,637)(1031,622)(1027,608)(1019,581)
\blacken\path(1024.327,704.578)(1019.000,581.000)(1081.855,687.533)(1024.327,704.578)
\put(997,431){\makebox(0,0)[b]{\smash{{\SetFigFont{10}{12.0}{\rmdefault}{\mddefault}{\updefault}$b$}}}}
\put(1747,806){\makebox(0,0)[b]{\smash{{\SetFigFont{10}{12.0}{\rmdefault}{\mddefault}{\updefault}$x$}}}}
\put(1747,131){\makebox(0,0)[b]{\smash{{\SetFigFont{10}{12.0}{\rmdefault}{\mddefault}{\updefault}$y$}}}}
\put(247,56){\makebox(0,0)[b]{\smash{{\SetFigFont{10}{12.0}{\rmdefault}{\mddefault}{\updefault}$c$}}}}
\path(149,184)(147,186)(142,190)
	(135,198)(126,207)(115,220)
	(103,234)(91,250)(79,269)
	(68,291)(56,318)(45,349)
	(37,377)(31,401)(27,421)
	(23,435)(20,445)(17,452)
	(15,457)(13,462)(12,469)
	(12,480)(13,495)(16,516)
	(21,543)(30,574)(40,601)
	(53,627)(66,650)(79,670)
	(93,689)(108,707)(122,723)
	(136,738)(150,751)(163,763)
	(174,773)(194,791)
\blacken\path(124.874,688.425)(194.000,791.000)(84.736,733.023)(124.874,688.425)
\end{picture}
}
\caption{}\label{FF5}
\end{figure}

\begin{figure}
\centering
\setlength{\unitlength}{0.00083333in}
\begingroup\makeatletter\ifx\SetFigFont\undefined%
\gdef\SetFigFont#1#2#3#4#5{%
  \reset@font\fontsize{#1}{#2pt}%
  \fontfamily{#3}\fontseries{#4}\fontshape{#5}%
  \selectfont}%
\fi\endgroup%
{\renewcommand{\dashlinestretch}{30}
\begin{picture}(1778,1678)(0,-10)
\put(902,55){\makebox(0,0)[b]{\smash{{\SetFigFont{10}{12.0}{\rmdefault}{\mddefault}{\updefault}$c$}}}}
\path(1007,1578)(1009,1578)(1013,1578)
	(1021,1577)(1032,1576)(1046,1575)
	(1064,1573)(1084,1571)(1105,1567)
	(1129,1564)(1153,1559)(1179,1553)
	(1206,1546)(1235,1537)(1265,1527)
	(1298,1514)(1332,1498)(1367,1480)
	(1401,1460)(1431,1441)(1457,1424)
	(1477,1409)(1493,1398)(1504,1388)
	(1512,1381)(1517,1375)(1522,1370)
	(1526,1365)(1531,1359)(1538,1351)
	(1547,1340)(1559,1326)(1575,1308)
	(1594,1286)(1615,1261)(1637,1233)
	(1660,1201)(1680,1172)(1696,1144)
	(1709,1119)(1719,1097)(1728,1075)
	(1735,1056)(1741,1037)(1746,1020)
	(1750,1005)(1753,992)(1757,970)
\blacken\path(1706.018,1082.698)(1757.000,970.000)(1765.050,1093.431)(1706.018,1082.698)
\path(1765,828)(1765,826)(1765,823)
	(1765,816)(1766,807)(1766,794)
	(1766,779)(1765,761)(1764,742)
	(1762,722)(1759,700)(1755,677)
	(1749,653)(1741,627)(1731,599)
	(1718,570)(1701,538)(1682,505)
	(1660,473)(1639,444)(1619,419)
	(1603,399)(1589,384)(1578,372)
	(1569,364)(1562,357)(1556,352)
	(1549,347)(1542,341)(1532,334)
	(1519,324)(1502,311)(1480,296)
	(1454,277)(1423,256)(1390,235)
	(1356,215)(1323,198)(1293,184)
	(1265,173)(1239,163)(1215,156)
	(1192,150)(1171,145)(1151,141)
	(1132,138)(1115,135)(1100,133)
	(1088,132)(1067,130)
\blacken\path(1183.615,171.242)(1067.000,130.000)(1189.304,111.512)(1183.615,171.242)
\path(745,123)(743,123)(739,124)
	(732,126)(721,128)(707,132)
	(689,136)(668,141)(646,148)
	(621,154)(595,162)(569,171)
	(541,180)(513,190)(484,202)
	(454,215)(423,231)(390,248)
	(358,267)(325,288)(291,312)
	(262,336)(238,356)(220,373)
	(207,386)(198,397)(193,405)
	(189,411)(187,417)(185,422)
	(183,429)(179,437)(174,448)
	(165,463)(154,481)(140,504)
	(124,530)(107,558)(88,594)
	(72,626)(61,654)(52,680)
	(45,702)(41,722)(37,741)
	(35,757)(34,772)(32,798)
\blacken\path(71.115,680.654)(32.000,798.000)(11.292,676.053)(71.115,680.654)
\path(55,985)(56,987)(57,992)
	(59,1000)(62,1011)(66,1026)
	(71,1044)(77,1063)(84,1084)
	(93,1106)(103,1130)(114,1154)
	(129,1181)(146,1209)(166,1239)
	(190,1270)(212,1297)(234,1320)
	(253,1340)(269,1356)(281,1368)
	(291,1377)(298,1383)(304,1388)
	(309,1391)(314,1395)(320,1398)
	(329,1404)(339,1411)(354,1420)
	(373,1432)(396,1447)(423,1463)
	(452,1480)(485,1498)(517,1513)
	(546,1525)(572,1535)(596,1543)
	(619,1549)(640,1555)(659,1559)
	(677,1562)(693,1565)(706,1567)(730,1570)
\blacken\path(614.648,1525.347)(730.000,1570.000)(607.206,1584.884)(614.648,1525.347)
\put(880,1540){\makebox(0,0)[b]{\smash{{\SetFigFont{10}{12.0}{\rmdefault}{\mddefault}{\updefault}$a$}}}}
\put(1750,843){\makebox(0,0)[b]{\smash{{\SetFigFont{10}{12.0}{\rmdefault}{\mddefault}{\updefault}$b$}}}}
\put(40,843){\makebox(0,0)[b]{\smash{{\SetFigFont{10}{12.0}{\rmdefault}{\mddefault}{\updefault}$x$}}}}
\put(1702.043,880.224){\arc{2108.136}{2.3997}{3.8179}}
\blacken\path(834.666,1424.914)(880.000,1540.000)(785.841,1459.787)(834.666,1424.914)
\end{picture}
}
\caption{}\label{FF6}
\end{figure}

\begin{enumerate}
\item In Figure \ref{FF5} there are two loop sets, $\{a,b,c\}$ and $\{b,x,y\}$.  The loop-buster $\{b\}$ is minimal absolute and $\{y,c\}$ is minimal relative. The loop set $\{y,b\}$ is not minimal absolute.
\item Consider Figure \ref{FF6}. There are two loops $\{a,b,c\}$ and $\{a,b,c,x\}$.  The minimal absolute loop-buster sets are $\{c\}, \{a\}$. $\{x,b\}$ is not minimal relative.
\end{enumerate}
\end{example}

\begin{example}[Loop-buster 2]\label{FE4}

Consider Figure \ref{FF7}.

\begin{figure}
\centering
\setlength{\unitlength}{0.00083333in}
\begingroup\makeatletter\ifx\SetFigFont\undefined%
\gdef\SetFigFont#1#2#3#4#5{%
  \reset@font\fontsize{#1}{#2pt}%
  \fontfamily{#3}\fontseries{#4}\fontshape{#5}%
  \selectfont}%
\fi\endgroup%
{\renewcommand{\dashlinestretch}{30}
\begin{picture}(3544,2771)(0,-10)
\put(3529,1514){\makebox(0,0)[b]{\smash{{\SetFigFont{10}{12.0}{\rmdefault}{\mddefault}{\updefault}$a_3$}}}}
\path(701,2264)(2381,217)
\blacken\path(2281.681,290.727)(2381.000,217.000)(2328.061,328.792)(2281.681,290.727)
\path(1278,209)(2861,2279)
\blacken\path(2811.935,2165.454)(2861.000,2279.000)(2764.274,2201.902)(2811.935,2165.454)
\path(739,2414)(741,2415)(745,2416)
	(752,2419)(763,2424)(777,2430)
	(795,2437)(816,2445)(839,2454)
	(864,2464)(891,2475)(919,2485)
	(947,2496)(977,2507)(1009,2519)
	(1041,2530)(1076,2542)(1112,2555)
	(1150,2567)(1188,2579)(1234,2593)
	(1275,2604)(1310,2614)(1341,2621)
	(1368,2627)(1392,2632)(1413,2636)
	(1432,2639)(1449,2642)(1464,2643)
	(1476,2645)(1496,2647)
\blacken\path(1379.581,2605.208)(1496.000,2647.000)(1373.610,2664.911)(1379.581,2605.208)
\path(1953,2662)(1955,2662)(1958,2662)
	(1964,2662)(1974,2661)(1986,2661)
	(2001,2660)(2019,2659)(2038,2658)
	(2059,2656)(2082,2654)(2107,2651)
	(2134,2647)(2163,2642)(2195,2636)
	(2231,2629)(2271,2620)(2313,2609)
	(2352,2598)(2389,2588)(2424,2577)
	(2457,2566)(2488,2555)(2517,2545)
	(2544,2535)(2570,2525)(2594,2515)
	(2618,2506)(2640,2496)(2660,2488)
	(2678,2480)(2693,2473)(2706,2468)(2726,2459)
\blacken\path(2604.258,2480.886)(2726.000,2459.000)(2628.880,2535.601)(2604.258,2480.886)
\path(2974,2309)(2976,2308)(2979,2306)
	(2985,2302)(2994,2296)(3006,2288)
	(3021,2278)(3038,2266)(3057,2252)
	(3077,2237)(3099,2221)(3121,2203)
	(3144,2185)(3168,2164)(3192,2142)
	(3217,2118)(3244,2091)(3271,2061)
	(3299,2029)(3326,1994)(3352,1959)
	(3375,1924)(3396,1892)(3414,1861)
	(3429,1832)(3443,1805)(3454,1779)
	(3465,1754)(3474,1731)(3482,1709)
	(3489,1687)(3495,1668)(3500,1651)
	(3504,1636)(3508,1624)(3513,1604)
\blacken\path(3454.791,1713.141)(3513.000,1604.000)(3513.000,1727.693)(3454.791,1713.141)
\path(3528,1424)(3528,1422)(3528,1418)
	(3529,1411)(3529,1401)(3530,1389)
	(3531,1373)(3531,1356)(3532,1338)
	(3532,1318)(3532,1297)(3531,1274)
	(3530,1250)(3528,1223)(3525,1193)
	(3521,1162)(3516,1127)(3511,1097)
	(3507,1074)(3505,1058)(3503,1048)
	(3502,1041)(3502,1036)(3501,1032)
	(3499,1026)(3497,1016)(3492,1001)
	(3484,980)(3474,953)(3461,922)
	(3447,892)(3432,865)(3418,841)
	(3405,821)(3392,802)(3379,786)
	(3367,772)(3355,759)(3345,747)(3326,727)
\blacken\path(3386.900,834.662)(3326.000,727.000)(3430.400,793.337)(3386.900,834.662)
\path(3124,524)(3123,523)(3120,520)
	(3115,514)(3108,507)(3098,497)
	(3087,486)(3074,473)(3060,459)
	(3044,445)(3026,430)(3007,414)
	(2984,396)(2959,378)(2930,358)
	(2899,337)(2870,319)(2842,303)
	(2816,288)(2791,275)(2767,263)
	(2745,252)(2725,243)(2705,234)
	(2686,226)(2668,218)(2652,212)
	(2638,206)(2626,202)(2606,194)
\blacken\path(2706.275,266.421)(2606.000,194.000)(2728.559,210.713)(2706.275,266.421)
\path(2178,82)(2176,82)(2172,81)
	(2165,79)(2154,77)(2140,74)
	(2123,71)(2103,67)(2081,63)
	(2058,59)(2034,55)(2008,51)
	(1980,48)(1951,45)(1919,42)
	(1885,40)(1849,38)(1811,37)
	(1769,37)(1731,38)(1696,40)
	(1666,43)(1638,46)(1614,49)
	(1591,53)(1570,57)(1551,61)
	(1534,65)(1520,68)(1496,74)
\blacken\path(1619.693,74.000)(1496.000,74.000)(1605.141,15.791)(1619.693,74.000)
\path(1083,157)(1081,158)(1077,160)
	(1069,163)(1058,168)(1045,174)
	(1028,181)(1010,189)(991,197)
	(970,207)(948,217)(925,228)
	(900,240)(873,254)(844,268)
	(814,284)(781,302)(753,317)
	(732,329)(718,338)(708,344)
	(703,348)(699,351)(696,353)
	(691,357)(683,363)(670,371)
	(652,382)(629,396)(603,412)
	(574,430)(548,446)(527,459)
	(509,470)(494,479)(480,488)
	(469,495)(446,509)
\blacken\path(564.102,472.232)(446.000,509.000)(532.905,420.980)(564.102,472.232)
\path(281,704)(280,705)(277,707)
	(272,711)(264,716)(255,724)
	(244,734)(232,745)(219,758)
	(206,774)(192,793)(178,814)
	(163,841)(148,872)(132,909)
	(116,952)(104,989)(93,1026)
	(84,1062)(76,1097)(69,1131)
	(63,1164)(58,1195)(53,1226)
	(50,1256)(46,1285)(43,1313)
	(40,1338)(38,1361)(36,1381)
	(35,1397)(33,1424)
\blacken\path(71.783,1306.544)(33.000,1424.000)(11.947,1302.112)(71.783,1306.544)
\path(86,1649)(87,1651)(88,1655)
	(91,1662)(96,1673)(101,1686)
	(108,1702)(116,1720)(124,1739)
	(134,1760)(144,1782)(156,1805)
	(169,1830)(185,1858)(202,1888)
	(221,1919)(241,1950)(259,1977)
	(274,1999)(286,2016)(295,2029)
	(302,2038)(307,2045)(312,2051)
	(316,2056)(321,2062)(328,2070)
	(337,2080)(349,2094)(364,2111)
	(382,2131)(401,2152)(425,2177)
	(446,2196)(464,2210)(478,2220)
	(490,2228)(501,2233)(510,2237)(528,2242)
\blacken\path(420.407,2180.977)(528.000,2242.000)(404.349,2238.788)(420.407,2180.977)
\put(1713,2609){\makebox(0,0)[b]{\smash{{\SetFigFont{10}{12.0}{\rmdefault}{\mddefault}{\updefault}$a_1$}}}}
\put(2891,2339){\makebox(0,0)[b]{\smash{{\SetFigFont{10}{12.0}{\rmdefault}{\mddefault}{\updefault}$a_2$}}}}
\put(626,2309){\makebox(0,0)[b]{\smash{{\SetFigFont{10}{12.0}{\rmdefault}{\mddefault}{\updefault}$a_9$}}}}
\put(1308,82){\makebox(0,0)[b]{\smash{{\SetFigFont{10}{12.0}{\rmdefault}{\mddefault}{\updefault}$a_6$}}}}
\put(2411,82){\makebox(0,0)[b]{\smash{{\SetFigFont{10}{12.0}{\rmdefault}{\mddefault}{\updefault}$a_5$}}}}
\put(3221,584){\makebox(0,0)[b]{\smash{{\SetFigFont{10}{12.0}{\rmdefault}{\mddefault}{\updefault}$a_4$}}}}
\put(63,1522){\makebox(0,0)[b]{\smash{{\SetFigFont{10}{12.0}{\rmdefault}{\mddefault}{\updefault}$a_8$}}}}
\put(386,569){\makebox(0,0)[b]{\smash{{\SetFigFont{10}{12.0}{\rmdefault}{\mddefault}{\updefault}$a_7$}}}}
\path(3326,1574)(310,1574)
\blacken\path(430.000,1604.000)(310.000,1574.000)(430.000,1544.000)(430.000,1604.000)
\end{picture}
}
\caption{}\label{FF7}
\end{figure}

The loops in this figure are many.  For example, we list some
\begin{enumerate}
\item $\{a_1,a_2,a_3,a_4,a_5,a_6,a_7,a_8,a_9\}$
\item $\{a_6,a_2,a_3,a_4,a_5\}$
\item $\{a_3,a_8,a_9,a_1,a_2\}$
\item $\{a_9,a_5,a_6,a_7,a_8\}$
\end{enumerate}
Consider the loop-buster
\[
\{a_2,a_5,a_8\}
\]
This is not a minimal absolute set but if we delete one of its elements we get a minimal absolute set. No one element is a loop-buster.
\end{example}

\begin{definition}[The loop-busting semantics LB$\CM$ for complete loops]\label{FD2}
Let $\CA=(S,R)$ be an argumentation network.  Assume that $( S,R)$ is a complete loop, namely that each of  its elements belongs to some loop cycle, as defined in item 1 of Definition \ref{FD1}.  We define the LB$\CM$ extensions for $\CA$ as follows.
\begin{enumerate}
\item Let $B$ be a loop-buster for $\CA$ satisfying $\CM$.
\item Let $Eq_{\max}(\CA)$ be the system of equations generated by $\CA$.  These have the form 
\[
({\bf eq}(x)): x =\Bh_x(y_1\comma y_{k(x)})
\]
where $x\in S$, and $y_1\comma y_{k(x)}$ are all the attackers of $x$. If $x$ has no attackers then $h_x \equiv 1$.
\item For each $x\in B$ replace the equation ${\bf eq}(x)$ by the two new equations 
\begin{itemize}
\item $({\bf eq}^*_a(x)): x=\Bh_x(y_1\comma y_{k(x)}) Z(x)$
\item $({\bf eq}^*_b(x)): Z(x)=0$
\end{itemize}
where $Z(x)$ is a new variable syntactically depending on $x$ alone.
\item Solve the equations in (3) and let $\Bf_B$ be any solution.  

Then the set 
\[
E_{{\Bf},B} =\{x\in S |\Bf_B(x)=1\}
\]
is an LB$\CM$ extension.
\item Thus the set of all LB$\CM$ extensions for $\CA=(S,R)$ is the set 
\[
\{E_{{\Bf},B}|B\mbox{ is as in (1), $\Bf_B$ is as in (4) and $E_{{\Bf},B}$ is as in (4)\}}
\]
\end{enumerate}
\end{definition}

Note that our definition of extension for a general network will be given in the next section.

Before we prove soundness of LB$\CM$ relative to the traditional Dung semantics and compare LB$\CM$ with CF2 semantics, let us do some examples. We use Figures \ref{FF2} and \ref{FF1}(b).

\begin{example}\label{FE5}
Consider Figure \ref{FF2}.  The only loop here is $\{\alpha, \beta,\phi\}$. There are three minimal absolute loop-busting sets, $B_\alpha =\{\alpha\}, B_\beta =\{\beta\}$ and $B_\phi =\{\phi\}$.

For each one of these sets we need to modify the equations of Figure \ref{FF2} and solve them and see what extensions we get. This has already been done in Example \ref{FE1}, parts (a), (b) and (c).

In (a) we made $\phi=0$, i.e.\ we used the loop-busting set $B_\phi$. We solved the modified equations and got the extension $\{\alpha,\delta\} =E_\phi$. In (b) we made $\alpha=0$, i.e.\ we used the set $B_\alpha$, solved the modified equations and got the extension $E_\alpha =\{\beta,\delta\}$.

In (c) we made $\beta=0$, i.e.\ we used the set $B_\alpha$, solved the modified equations and got the extension $E_\beta=\{\phi,\delta\}$. 

Let us now compare with the CF2 extensions for the figure (Figure \ref{FF2}).  The maximal conflict free sets of the first loop $\{\alpha,\beta,\phi\}$ are $C_\alpha =\{\alpha\}, C_\beta =\{\beta\}$ and $C_\phi=\{\phi\}$.  They are the same as our loop-busting sets, but they are used differently. They are supposed to be in (i.e.\ value 1) not out (value 0). We use $C_\alpha, C_\beta, C_\phi$ to calculate the CF2 extensions and get $\{\alpha,\delta\}, \{\beta,\delta\}$ and $\{\phi,\delta\}$, indeed the same as the LB extensions.
\end{example}

\begin{example}\label{FE6}
We now consider Figure \ref{FF1}(b).  The only minimal absolute loop-buster set here is $B_\gamma=\{\gamma\}$.  We have three more minimal relative sets, $B_1=\{\beta,\phi\}, B_2=\{\delta,\phi\}$ and $B_3=\{\alpha,\phi\}$.

We refer the reader to Example \ref{FE2}, where some equational calculations for this figure are carried out.  
\begin{enumerate}
\item In (a) of Example \ref{FE2}, we make $\gamma =0$, we solve the modified equation and get the extension $E_\gamma =\{\alpha,\beta,\phi\}$.

This takes care of the case $B_\gamma=\{\gamma\}$.
\item Let us address the case of $B_3=\{\alpha,\phi\}$.  We use (b) of Example \ref{FE2}, where we make $\alpha=0$.  We modify the equation for $\alpha$ and get a solution $\alpha=0, \delta=1, \beta=0, \gamma=1$ and $\phi=0$.

We needed to also make $\phi=0$ for the loop-buster set $\{\alpha,\phi\}$, but as it turns out, making $\alpha=0$ also makes $\phi=0$.  We thus get the extension $E_{\alpha,\phi} =\{\delta,\gamma\}$.
\item Let us address the case of $B_1=\{\beta,\phi\}$.  This corresponds to case (d) $\beta=0$ of Example \ref{FE2}.  We modify the equations and solve them and get $\beta=0, \gamma=1,\phi=0, \alpha=0$ and $\delta=1$.

The extension is $\{\gamma,\delta\}$. 

Again, although we did not explicitly make the requirement $\phi=0$, the equations obtained from the requirement $\alpha=0$ did the job for us.
\item We now check the case of $B_2=\{\delta,\phi\}$. Here we get a discrepancy with case (c) of Example \ref{FE2}.

There, in case (c), we only require $\delta=0$, solve the equations and get the extension $\{\alpha,\beta,\phi\}$. This is not what we want, as we also require $\phi=0$. So let us do the calculation in detail here.

The modified equation system for $B_1 =\{\delta,\phi\}$ is the following:
\begin{itemlist}{xxxx}
\item [1.] $\alpha =1-\gamma$
\item [$2^*a$.] $\delta =(1-\alpha) Z(\delta)$
\item [$2^*b$.] $Z(\delta)=0$.
\item [3.] $\beta =1-\delta$
\item [5.] $\gamma=1-\beta$
\item [$5^*a$.] $\gamma =\phi =(1-\gamma) Z(\phi)$
\item [$5^*b$.] $Z(\phi)=0$.
\end{itemlist}
We solve the equations and get $\phi =0, \delta=0, \beta=1, \gamma=0, \alpha=1$.

The extension is $\{\alpha,\beta\}$.
\end{enumerate}
\end{example}

\begin{example}[CF2 and the LB minimal absolute semantics]\label{EFeb05-1}
The LB minimal absolute semantics does not give all the CF2 extensions in the case of even loops. Consider Figure \ref{FF4b}. The set $B =\{a_{10}\}$ yields the extension $E_B =\{b,a_1,a_3,a_5,a_7,a_0\}$. $B$ is minimal absolute. Consider now $B'$ being $B'=\{a_{10}, a_3,a_6\}$.

This yields 
\[
E_{B'} =\{a_1,a_4,a_7,a_9\}
\]
However, $B'$ is not minimal absolute. $E_{B'}$ is a CF2 extension. $B'$ is a  minimal relative set.

What happens here is that the minimal absolute semantics gives the same extensions for even loops as the traditional Dung extensions, but the CF2 semantics gives more. This is a weakness of CF2.
\end{example}

\begin{remark}[CF2 and the minimal relative extensions]\label{FR1}
Let us discuss the results of Example \ref{FE6} calculated for Figure \ref{FF1}(b) and compare them with the CF2 extensions of Figure \ref{FF1}(b).  This will give us an idea about the relation of CF2 to  the minimal relative semantics. We use Example \ref{FE2}, where all the extensions were calculated and especially refer to Table \ref{A}, given in item (g) of Example \ref{FE2}, which summarises these calculations.
\begin{enumerate}
\item The CF2 extensions are all the conflict free subsets. These are $\{\delta,\phi\}, \{\delta,\gamma\}, \{\alpha,\beta,\phi\}$.

Comparing with the  semantics of Table \ref{A}, we get the following: the LB minimal absolute extensions are one only, namely $\{\delta,\gamma\}$. The LB minimal relative extensions are $\{\delta,\gamma\}, \{\alpha,\beta,\phi\}$ and $\{\alpha,\beta\},\{\gamma,\phi\}$.

We see that LB minimal absolute gives less extensions (but breaks loops) while LB minimal relative gives one more extension. Obviously we need to identify a policy $\CM$ which will yield exactly  the CF2 extensions.
\item Let us examine case (4) of Example \ref{FE6} more closely.  This is the case of $B_2 =\{\delta,\phi\}$ of Figure \ref{FF1}(b).  The loop-buster set $B_2$ was introduced to bust two loops. The loop $\{\alpha,\beta,\gamma,\delta\}$ and the loop $\{\gamma,\phi\}$.  $\delta$ was included to bust the first loop and $\gamma$ was included to bust the second loop.  Our equational computations show in case (c) of Example \ref{FE2} that if we start with $\delta=0$ we get that it follows that $\gamma=0$. But $\gamma$ belongs also to the second loop $\{\gamma,\phi\}$. So $\{\delta\}$ on its own is a loop-buster for both loops and we do not need to include $\phi$ in the loop-buster.  So $B_1=\{\delta,\phi\}$ is not minimal relative because $B'_2=\{\delta\}$ can do the job.     The above considerations show that the definition of minimal relative loop-busting sets needs to be adjusted.  This needs to be done in a methodologically correct manner and will be addressed in the next section. 
\end{enumerate}
Note that if we accept that $B'_2 =\{\delta\}$ is the minimal relative loop-busting set, then the calculated extension for this case is $E_\delta =\{\delta,\phi\}$, in complete agreement with the CF2 semantics!
\end{remark}

We now need to demonstrate the soundness of the LB semantics. The perceptive reader will ask himself, how do the LB extensions relate to the extensions of traditional Dung semantics?  After all, we start with the standard equational semantics, which for the case of $Eq_{\max}$ is identical with the Dung semantics, but then using a loop-busting set $B$ of one kind or another, we get a new set of equations and call the solutions LB extensions.  What are these solutions and what meaning can we give them?

Obviously, we need some sort of soundness result.  This is the job of the next theorem.

\begin{theorem}[Representation theorem for LB semantics]\label{FT2}
Let $\CA =(S,R)$ be an argumentation net being a complete loop as in Definition \ref{FD2} and let $B$ be a loop-busting subset of $S$ (of some sort $\CM$).  Let $\BE(B,\CA)$ be the family of LB extensions obtained from $\CA$ and $B$ by following the procedures of Definition \ref{FD2}.  Then $\BE(B,\CA)$ can be obtained also following the procedure below
\begin{enumerate}
\item For each $x \in B$, let $\Bz(x)$ be a new point not in $S$. Let $\Bz(x)$ be all different for different $x$s.
\item Define $(S_B, R_B)$ as follows:
\[\begin{array}{l}
S_B=S\cup \{\Bz(x)|x\in B\}\\
R_B=R\cup \{(\Bz(x),x)|x\in B\}.
\end{array}\]
\item The network $(S_B, R_B)$ is an ordinary Dung network and has traditional Dung extensions. We have (for $Eq_{\max}$): 
\[
\BE(B,\CA)=\{E\cap S | E\mbox{ is an extension of }(S_B,R_B)\}
\]
\end{enumerate}
\end{theorem}

\begin{proof}
The new equations for each $x\in B$ in $(S_B,R_B)$ are
\[\begin{array}{l}
({\bf eq}^*_a(x)): x=1-\max(y_1\comma y_{k(x)},\Bz(x))\\
({\bf eq}^*_a(\Bz(x))):\Bz(x)=1
\end{array}\]
where $y_1\comma y_{k(x)}$ are all the attackers of $x$ in $(S,R)$.

Since $\Bz(x)=1$, we get that 
\[\begin{array}{rcl}
1-\max(y_2\comma y_{k(x)},\Bz(x)) &=& (1-\max(y_1\comma y_{k(x)})(1-\Bz(x))\\
&=& (1-\max(y_1\comma y_{k(x)} )) Z(x)
\end{array}
\]
provided $Z(x) =1-\Bz(x)$.

Of course $\Bz(x)=1$ means $Z(x)=0$.

So we get the same modified equations as required by the LB semantics in Definition \ref{FD2}.
\end{proof}

\begin{example}\label{FE7}
Let us represent the cases of Example \ref{FE6}, which dealt with Figure \ref{FF1}(b). See Figures \ref{FF8}, \ref{FF9}, \ref{FF10}, \ref{FF11} corresponding to cases (1)--(4) of Example \ref{FE6}.

We are also adding Figure \ref{FF12}, describing the situation for $B'_2 =\{\delta\}$ as discussed in Remark \ref{FR1} in item (b).
\end{example}

\begin{figure}
\centering
\setlength{\unitlength}{0.00083333in}
\begingroup\makeatletter\ifx\SetFigFont\undefined%
\gdef\SetFigFont#1#2#3#4#5{%
  \reset@font\fontsize{#1}{#2pt}%
  \fontfamily{#3}\fontseries{#4}\fontshape{#5}%
  \selectfont}%
\fi\endgroup%
{\renewcommand{\dashlinestretch}{30}
\begin{picture}(2730,1639)(0,-10)
\put(2565,1489){\makebox(0,0)[b]{\smash{{\SetFigFont{10}{12.0}{\rmdefault}{\mddefault}{\updefault}$\Bz(\gamma)$}}}}
\put(1927.500,2576.500){\arc{4275.658}{1.1935}{1.9480}}
\blacken\path(1263.040,576.310)(1140.000,589.000)(1242.593,519.901)(1263.040,576.310)
\path(2565,1414)(1065,814)
\blacken\path(1165.275,886.421)(1065.000,814.000)(1187.559,830.713)(1165.275,886.421)
\path(615,1189)(1065,814)
\blacken\path(953.608,867.775)(1065.000,814.000)(992.019,913.869)(953.608,867.775)
\blacken\path(542.019,1089.131)(615.000,1189.000)(503.608,1135.225)(542.019,1089.131)
\path(615,1189)(165,814)
\path(615,214)(165,589)
\blacken\path(276.392,535.225)(165.000,589.000)(237.981,489.131)(276.392,535.225)
\blacken\path(687.981,313.869)(615.000,214.000)(726.392,267.775)(687.981,313.869)
\path(615,214)(1065,589)
\put(615,1264){\makebox(0,0)[b]{\smash{{\SetFigFont{10}{12.0}{\rmdefault}{\mddefault}{\updefault}$\beta$}}}}
\put(615,64){\makebox(0,0)[b]{\smash{{\SetFigFont{10}{12.0}{\rmdefault}{\mddefault}{\updefault}$\alpha$}}}}
\put(1065,664){\makebox(0,0)[b]{\smash{{\SetFigFont{10}{12.0}{\rmdefault}{\mddefault}{\updefault}$\gamma$}}}}
\put(15,664){\makebox(0,0)[b]{\smash{{\SetFigFont{10}{12.0}{\rmdefault}{\mddefault}{\updefault}$\delta$}}}}
\put(2715,664){\makebox(0,0)[b]{\smash{{\SetFigFont{10}{12.0}{\rmdefault}{\mddefault}{\updefault}$\phi$}}}}
\put(1927.500,-1173.500){\arc{4275.658}{4.3351}{5.0896}}
\blacken\path(2591.960,826.690)(2715.000,814.000)(2612.407,883.099)(2591.960,826.690)
\end{picture}
}
\caption{Case (1): $\gamma=0$ for $B_\gamma=\{0\}$}\label{FF8}
\end{figure}
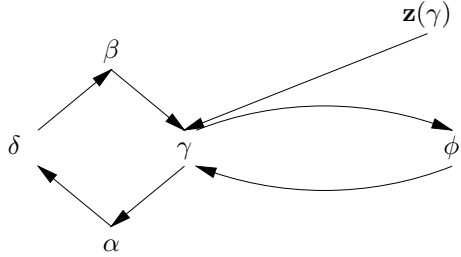

\begin{figure}
\centering
\setlength{\unitlength}{0.00083333in}
\begingroup\makeatletter\ifx\SetFigFont\undefined%
\gdef\SetFigFont#1#2#3#4#5{%
  \reset@font\fontsize{#1}{#2pt}%
  \fontfamily{#3}\fontseries{#4}\fontshape{#5}%
  \selectfont}%
\fi\endgroup%
{\renewcommand{\dashlinestretch}{30}
\begin{picture}(2880,2326)(0,-10)
\put(2865,2164){\makebox(0,0)[b]{\smash{{\SetFigFont{10}{12.0}{\rmdefault}{\mddefault}{\updefault}$\Bz(\phi)$}}}}
\put(1927.500,3251.500){\arc{4275.658}{1.1935}{1.9480}}
\blacken\path(1263.040,1251.310)(1140.000,1264.000)(1242.593,1194.901)(1263.040,1251.310)
\path(615,1864)(1065,1489)
\blacken\path(953.608,1542.775)(1065.000,1489.000)(992.019,1588.869)(953.608,1542.775)
\blacken\path(542.019,1764.131)(615.000,1864.000)(503.608,1810.225)(542.019,1764.131)
\path(615,1864)(165,1489)
\path(615,889)(165,1264)
\blacken\path(276.392,1210.225)(165.000,1264.000)(237.981,1164.131)(276.392,1210.225)
\blacken\path(687.981,988.869)(615.000,889.000)(726.392,942.775)(687.981,988.869)
\path(615,889)(1065,1264)
\path(1365,214)(690,664)
\blacken\path(806.487,622.397)(690.000,664.000)(773.205,572.474)(806.487,622.397)
\path(2790,2089)(2790,1489)
\blacken\path(2760.000,1609.000)(2790.000,1489.000)(2820.000,1609.000)(2760.000,1609.000)
\put(615,1939){\makebox(0,0)[b]{\smash{{\SetFigFont{10}{12.0}{\rmdefault}{\mddefault}{\updefault}$\beta$}}}}
\put(615,739){\makebox(0,0)[b]{\smash{{\SetFigFont{10}{12.0}{\rmdefault}{\mddefault}{\updefault}$\alpha$}}}}
\put(1065,1339){\makebox(0,0)[b]{\smash{{\SetFigFont{10}{12.0}{\rmdefault}{\mddefault}{\updefault}$\gamma$}}}}
\put(15,1339){\makebox(0,0)[b]{\smash{{\SetFigFont{10}{12.0}{\rmdefault}{\mddefault}{\updefault}$\delta$}}}}
\put(2715,1339){\makebox(0,0)[b]{\smash{{\SetFigFont{10}{12.0}{\rmdefault}{\mddefault}{\updefault}$\phi$}}}}
\put(1440,64){\makebox(0,0)[b]{\smash{{\SetFigFont{10}{12.0}{\rmdefault}{\mddefault}{\updefault}$\Bz(\alpha)$}}}}
\put(1927.500,-498.500){\arc{4275.658}{4.3351}{5.0896}}
\blacken\path(2591.960,1501.690)(2715.000,1489.000)(2612.407,1558.099)(2591.960,1501.690)
\end{picture}
}
\caption{Case (2): $B_3=\{\alpha,\phi\}$}\label{FF9}
\end{figure}
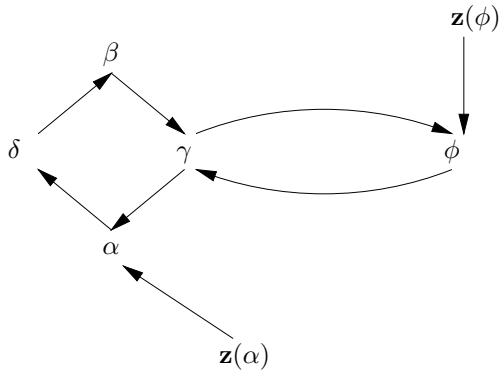

\begin{figure}
\centering
\setlength{\unitlength}{0.00083333in}
\begingroup\makeatletter\ifx\SetFigFont\undefined%
\gdef\SetFigFont#1#2#3#4#5{%
  \reset@font\fontsize{#1}{#2pt}%
  \fontfamily{#3}\fontseries{#4}\fontshape{#5}%
  \selectfont}%
\fi\endgroup%
{\renewcommand{\dashlinestretch}{30}
\begin{picture}(2880,2014)(0,-10)
\put(2865,1489){\makebox(0,0)[b]{\smash{{\SetFigFont{10}{12.0}{\rmdefault}{\mddefault}{\updefault}$\Bz(\phi)$}}}}
\put(1927.500,2576.500){\arc{4275.658}{1.1935}{1.9480}}
\blacken\path(1263.040,576.310)(1140.000,589.000)(1242.593,519.901)(1263.040,576.310)
\path(615,1189)(1065,814)
\blacken\path(953.608,867.775)(1065.000,814.000)(992.019,913.869)(953.608,867.775)
\blacken\path(542.019,1089.131)(615.000,1189.000)(503.608,1135.225)(542.019,1089.131)
\path(615,1189)(165,814)
\path(615,214)(165,589)
\blacken\path(276.392,535.225)(165.000,589.000)(237.981,489.131)(276.392,535.225)
\blacken\path(687.981,313.869)(615.000,214.000)(726.392,267.775)(687.981,313.869)
\path(615,214)(1065,589)
\path(2790,1414)(2790,814)
\blacken\path(2760.000,934.000)(2790.000,814.000)(2820.000,934.000)(2760.000,934.000)
\path(615,1789)(615,1414)
\blacken\path(585.000,1534.000)(615.000,1414.000)(645.000,1534.000)(585.000,1534.000)
\put(615,1264){\makebox(0,0)[b]{\smash{{\SetFigFont{10}{12.0}{\rmdefault}{\mddefault}{\updefault}$\beta$}}}}
\put(615,64){\makebox(0,0)[b]{\smash{{\SetFigFont{10}{12.0}{\rmdefault}{\mddefault}{\updefault}$\alpha$}}}}
\put(1065,664){\makebox(0,0)[b]{\smash{{\SetFigFont{10}{12.0}{\rmdefault}{\mddefault}{\updefault}$\gamma$}}}}
\put(15,664){\makebox(0,0)[b]{\smash{{\SetFigFont{10}{12.0}{\rmdefault}{\mddefault}{\updefault}$\delta$}}}}
\put(2715,664){\makebox(0,0)[b]{\smash{{\SetFigFont{10}{12.0}{\rmdefault}{\mddefault}{\updefault}$\phi$}}}}
\put(615,1864){\makebox(0,0)[b]{\smash{{\SetFigFont{10}{12.0}{\rmdefault}{\mddefault}{\updefault}$\Bz(\beta)$}}}}
\put(1927.500,-1173.500){\arc{4275.658}{4.3351}{5.0896}}
\blacken\path(2591.960,826.690)(2715.000,814.000)(2612.407,883.099)(2591.960,826.690)
\end{picture}
}
\caption{Case (3): $B_1=\{\beta,\phi\}$}\label{FF10}
\end{figure}
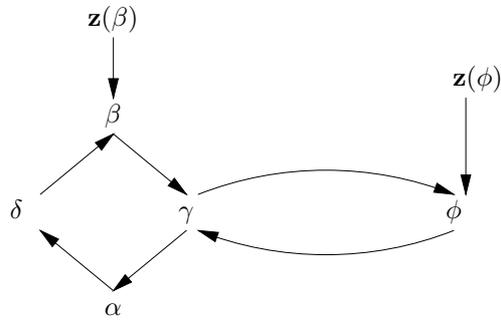

\begin{figure}
\centering
\setlength{\unitlength}{0.00083333in}
\begingroup\makeatletter\ifx\SetFigFont\undefined%
\gdef\SetFigFont#1#2#3#4#5{%
  \reset@font\fontsize{#1}{#2pt}%
  \fontfamily{#3}\fontseries{#4}\fontshape{#5}%
  \selectfont}%
\fi\endgroup%
{\renewcommand{\dashlinestretch}{30}
\begin{picture}(2907,1789)(0,-10)
\put(42,1639){\makebox(0,0)[b]{\smash{{\SetFigFont{10}{12.0}{\rmdefault}{\mddefault}{\updefault}$\Bz(\delta)$}}}}
\put(1954.500,2576.500){\arc{4275.658}{1.1935}{1.9480}}
\blacken\path(1290.040,576.310)(1167.000,589.000)(1269.593,519.901)(1290.040,576.310)
\path(642,1189)(1092,814)
\blacken\path(980.608,867.775)(1092.000,814.000)(1019.019,913.869)(980.608,867.775)
\blacken\path(569.019,1089.131)(642.000,1189.000)(530.608,1135.225)(569.019,1089.131)
\path(642,1189)(192,814)
\path(642,214)(192,589)
\blacken\path(303.392,535.225)(192.000,589.000)(264.981,489.131)(303.392,535.225)
\blacken\path(714.981,313.869)(642.000,214.000)(753.392,267.775)(714.981,313.869)
\path(642,214)(1092,589)
\path(2817,1414)(2817,814)
\blacken\path(2787.000,934.000)(2817.000,814.000)(2847.000,934.000)(2787.000,934.000)
\path(42,1564)(42,814)
\blacken\path(12.000,934.000)(42.000,814.000)(72.000,934.000)(12.000,934.000)
\put(642,1264){\makebox(0,0)[b]{\smash{{\SetFigFont{10}{12.0}{\rmdefault}{\mddefault}{\updefault}$\beta$}}}}
\put(642,64){\makebox(0,0)[b]{\smash{{\SetFigFont{10}{12.0}{\rmdefault}{\mddefault}{\updefault}$\alpha$}}}}
\put(1092,664){\makebox(0,0)[b]{\smash{{\SetFigFont{10}{12.0}{\rmdefault}{\mddefault}{\updefault}$\gamma$}}}}
\put(42,664){\makebox(0,0)[b]{\smash{{\SetFigFont{10}{12.0}{\rmdefault}{\mddefault}{\updefault}$\delta$}}}}
\put(2742,664){\makebox(0,0)[b]{\smash{{\SetFigFont{10}{12.0}{\rmdefault}{\mddefault}{\updefault}$\phi$}}}}
\put(2892,1489){\makebox(0,0)[b]{\smash{{\SetFigFont{10}{12.0}{\rmdefault}{\mddefault}{\updefault}$\Bz(\phi)$}}}}
\put(1954.500,-1173.500){\arc{4275.658}{4.3351}{5.0896}}
\blacken\path(2618.960,826.690)(2742.000,814.000)(2639.407,883.099)(2618.960,826.690)
\end{picture}
}
\caption{Case (4): $B_2=\{\delta,\phi\}$}\label{FF11}
\end{figure}
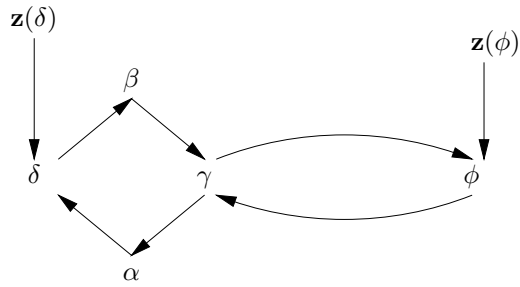

\begin{figure}
\centering
\setlength{\unitlength}{0.00083333in}
\begingroup\makeatletter\ifx\SetFigFont\undefined%
\gdef\SetFigFont#1#2#3#4#5{%
  \reset@font\fontsize{#1}{#2pt}%
  \fontfamily{#3}\fontseries{#4}\fontshape{#5}%
  \selectfont}%
\fi\endgroup%
{\renewcommand{\dashlinestretch}{30}
\begin{picture}(2757,1789)(0,-10)
\put(42,1639){\makebox(0,0)[b]{\smash{{\SetFigFont{10}{12.0}{\rmdefault}{\mddefault}{\updefault}$\Bz(\delta)$}}}}
\put(1954.500,2576.500){\arc{4275.658}{1.1935}{1.9480}}
\blacken\path(1290.040,576.310)(1167.000,589.000)(1269.593,519.901)(1290.040,576.310)
\path(642,1189)(1092,814)
\blacken\path(980.608,867.775)(1092.000,814.000)(1019.019,913.869)(980.608,867.775)
\blacken\path(569.019,1089.131)(642.000,1189.000)(530.608,1135.225)(569.019,1089.131)
\path(642,1189)(192,814)
\path(642,214)(192,589)
\blacken\path(303.392,535.225)(192.000,589.000)(264.981,489.131)(303.392,535.225)
\blacken\path(714.981,313.869)(642.000,214.000)(753.392,267.775)(714.981,313.869)
\path(642,214)(1092,589)
\path(42,1564)(42,814)
\blacken\path(12.000,934.000)(42.000,814.000)(72.000,934.000)(12.000,934.000)
\put(642,1264){\makebox(0,0)[b]{\smash{{\SetFigFont{10}{12.0}{\rmdefault}{\mddefault}{\updefault}$\beta$}}}}
\put(642,64){\makebox(0,0)[b]{\smash{{\SetFigFont{10}{12.0}{\rmdefault}{\mddefault}{\updefault}$\alpha$}}}}
\put(1092,664){\makebox(0,0)[b]{\smash{{\SetFigFont{10}{12.0}{\rmdefault}{\mddefault}{\updefault}$\gamma$}}}}
\put(42,664){\makebox(0,0)[b]{\smash{{\SetFigFont{10}{12.0}{\rmdefault}{\mddefault}{\updefault}$\delta$}}}}
\put(2742,664){\makebox(0,0)[b]{\smash{{\SetFigFont{10}{12.0}{\rmdefault}{\mddefault}{\updefault}$\phi$}}}}
\put(1954.500,-1173.500){\arc{4275.658}{4.3351}{5.0896}}
\blacken\path(2618.960,826.690)(2742.000,814.000)(2639.407,883.099)(2618.960,826.690)
\end{picture}
}
\caption{Case (b): $B'_2=\{\delta\}$}\label{FF12}
\end{figure}
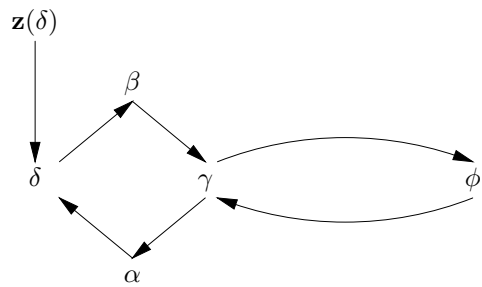

\section{The equational semantics LB and its connection with CF2}
We now define the family of LB semantics and identify the loop-busting counterpart of CF2.  We need to develop some concepts first.  We begin with a high school example.

\begin{example}[High school example]\label{FE8}
\begin{enumerate}
\item Solve the following equations in the unknowns $x,y,z$.
\begin{enumerate}
\item $x-y=1$
\item $x+y=5$
\item $z^2-4yz +x+1=0$
\end{enumerate}
The point I want to make is that we solve the equations directionally. We first find the values of $x$ and $y$ from equations (a) and (b) to be $x=3$ and $y=1$ and then substitute in equation (c) and solve it.  We get
\begin{enumerate}
\setcounter{enumii}{2}
\item $z^2-4z+4=0$\\
$z=2$
\end{enumerate}
\item Let us change the problem a bit.  We have the equations
\begin{enumerate}
\item $x-\sin y = 2.99$
\item $x+\sin y = 3.01$
\item $z^2-400 yz +x+1=0$ 
\end{enumerate}
Here we may again consider equations (a) and (b) first but also use the approximation $y\approx \sin y$. We find $x=3, y\approx 0.01$ and solve the third to get $z=2$.
\item A third possibility is to look at equations (a) and (b) and decide to ignore them altogether,\footnote{Of course, ignoring (a) and (b) needs to be justified.} and substitute $x=y=0$. We get 
\begin{enumerate}
\item [(c$^*$).] $z^2+1=0$
\end{enumerate}
\item Another example is the equation 
\[
x^4-2x^2+\frac{x}{1000}+1=0.
\]
To solve this equation we decide on the perturbation which ignores $\frac{x}{1000}$ on account of it being relatively small. We solve 
\[
x^4-2x^2+1=0
\]
we get $x =\pm 1$.
\end{enumerate}
\end{example}

\begin{remark}\label{RFeb06-1}
We present a perturbation protocol for solving equations of the form 
\[
x =\Bh_x (v_1,\ldots).
\]
\begin{enumerate}
\item Let $V$ be a set of variables and $\BBE$ be a set of equations of the form 
$x =\Bh_x(V_x)$, where $V_x\subseteq V$ are the variables appearing in $\Bh_x$, and 
$x$ ranges over $V$.  We seek solutions to the system $\BBE$ with values hopefully in 
$\{0,1\}$. If $\Bh_x$ are all continuous functions in $[0,1]$, then we know that  there 
are solutions with values in $[0,1]$, but are there solutions with values in $\{0,1\}$?

Even if we are looking for and happy with any kind of solution, we may wish to shorten the computation by starting with some good guesses, or some approximation or follow any kind of protocol $\BBP$ which will enable us to perturb the equations and get some results which we would find satisfactory from the point of view of our application area. 

In the case of equations arising from argumentation networks, we would like perturbations which help us overcome odd-numbered loops.

Note that in numerical analysis such equations are well known. If $x_1\comma x_m$ are variables in $[0,1]$ and $\Bh_1\comma \Bh_m$ are continuous functions in $[0,1]$, we want to solve the equations 
\[
x_i=\Bh_i(x_1\comma x_m), i=1\comma m.
\]
One well known method is that of successive approximations. We guess a starting value 
\[
x_1=a^0_1\comma x_m=a^0_m
\]
and continue by substituting
\[
a^{j+1}_i=\Bh_i(a^j_1\comma a^j_m).
\]
Under certain conditions on the functions $\Bh_i$ (Lipschitz condition), the values $a^j_i, =1,2,\ldots$ converge to a limit $a^\infty_i, i=1,2\comma m$ and that would be a solution.  What we are going to do in this paper is in the same spirit.

\item Let us proceed formally adopting a purely equational point of view and take a subset $B_1\subseteq V$ of the variables and decide for our own reasons to substitute the value $0$ for all the variables in $B_1$ in the equations $\BBE$.

How we choose $B_1$ is not said here, we assume that we have some protocols for finding such a $B_1$.  In the application area of argumentation, these protocols will be different loop-busting protocols LB$(\CM)$.

For the moment, formally from the equational point of view, we have a set of equations $\BBE$ with variables $V$ and a $B_1\subseteq V$, which we want to make 0.  How do we proceed?

This has to be done carefully and so we replace for each $u \in B_1$, the equation 
\[
\Be{\bf q}(u):u =\Bh_u(V_u)
\]
by the pair of equations 
\[
{\bf eq}^*(u): \left\{
\begin{array}{l}
u =\Bh_u(V_u) Z(u)\\
Z(u)=0
\end{array}
\right.
\]
We now propagate these values through the new set of equations, solve what we can solve and end up with new equations of the form 
\[
{\bf eq}^1(x): x=\Bh^1_x(V^1_x)
\]
for $x \in V$, where $\Bh^1_x$ is the new equation for $x$ and $V^1_x$ are its variables. We have 
\[
V^1_x\subseteq V-B_1.
\]
The variables of $B_1$ get all value 0 and maybe more variables solve to some numerical values.  Note that we can allow also for the case of $B_1=\varnothing$.

We always have a solution because the functions involved are all continuous.

Let $U_1=\{x|V^1_x=\varnothing\}$. $U_1$ is the set of $x$ which get a definite numerical value, for which $V_x$, the set of variables they depend on, is empty. We have $B_1\subseteq U_1\subseteq V$.

Let $\Bf_1$ be a function collecting these values on $U_1$, i.e.\ $\Bf(x)=\Bh^1_x$, for $x \in U_1$.

\item We refer to $U_1$ as the set of all elements instantiated to numerical values at step 1. We declare all  variables of $U_1$ as having rank 1.

\item Let $\BBE^1$ be the system of equations for the variables in $V-U_1$.

We now have a new system of variables and we can repeat the procedure by using a new set $B_2$ chosen to make $0$.

We can carry this procedure repeatedly until we get numerical values for all variables. Say that at step $n$ we have that the union of all sets $U_1, U_2\comma  U_n$  equals $V$. Then also each element of $V$ has a clear rank $k$, the step at which $x$ was instantiated. Call this procedure Protocol $\BBP=(B_1,B_2,\ldots)$.  Note that we did not say why and how we choose the sets $B_1,B_2,\ldots$. In the case of equations arising from argumentation networks, these sets $B_i$ will be loop-busting sets.
\item Note that the equations initially give variables either 0 or 1 and our loop busters also give variables $o$, and $Eq_{\max }$ and $Eq_{\rm inverse}$ are such that they keep the variables in $\{0,1\}$, then all the functions \Bf\ involved are $\{0,1\}$ functions
\end{enumerate}
\end{remark}

\begin{example}\label{EFeb07-1}
We now explain why we use $Z$ in our perturbation. Consider the equations
\begin{enumerate}
\item $a=1-c$
\item $b=1-a$
\item $c=1-b$.
\end{enumerate}
These equations correspond to a 3-element argumentation loop.  

We take the $B_1=\{a\}$ and want to execute a perturbation. If we do  just substitute $a=0$, we get a contradiction because the equations prove algebraically through manipulation that 
\[
a=b=c=1-a=1-b=1-c
\]
So we need to change the equation governing $a$.  We write
\begin{itemlist}{xxxx}
\item [$1^*a$.] $a=(1-c)Z(a)$
\item [$1^*b$.] $Z(a)=0$
\end{itemlist}
Algebraically we now have 4 equations in 4 variables
\[
a,b,c,Z(a)
\]
The solution is 
\[
Z(a)=0, a=0, b=1, c=0.
\]
We cannot any more execute an algebraic manipulation to get $a=b=c$!
\end{example}

\begin{example}\label{EFeb07-2}
Let us recall Example \ref{FE1}, manipulating the equations arising from  Figure \ref{FF2}.  This is an illustration of our procedure. We used the loop-busting sets $B_\alpha =\{\alpha\}, B_\beta=\{\beta\}$ and $B_\phi=\{\phi\}$, and followed the procedure as described in Remark \ref{RFeb06-1}.
\end{example}

Let us now proceed with more concepts leading the way to the full definition of our loop-busting LB semantics.  

We saw how to get a set of equations $Eq(\CA)$ from any argumentation  network 
$\CA$. Now we want to show how to get an argumentation  network $\CB_{\BBE}$ 
from any set of equations $\BBE$.

Furthermore, once we have a set of equations $\CB_{\BBE}$, we can perturb it to get a new set of equations $\BBE_B$ using some perturbation set $B$ and then from the equations $\BBE_B$ get a new argumentation network $\CA_{\BBE_B}$. The net result of all these steps is that we start with a network $\CB=(S,R)$ and a perturbation set of nodes $B\subseteq S$ and we end up with a new network which we can denote by $\CA=\CB_B$. If $B$ is a loop-busting set, then $\CA$ is the loop-busted result of applying $B$ to $\CB$.

\begin{definition}\label{DFeb07-3}
\begin{enumerate}
\item Let $V$ be a set of variables and let $x=\Bh_x(V), x\in V, V_x\subseteq V$ be a system of equations $\BBE$, where $V_x$ is the set of variables actually appearing in $\Bh_x$. We now define the associated argumentation network $\CA_{\BBE} =(S_{\BBE}, R_{\BBE})$ as follows:
\begin{enumerate}
\item Let $S_{\BBE}=V$
\item Let $yR_{\BBE}x$ hold iff $y \in V_x$.\footnote{The definition of $yRx$ as $y\in V_x$ is a very special definition, making essential use of the fact that the equation 
\[
x=\Bh_x(V_x)
\]
is of a very special form of either 
\[
x =1-\max V_x
\]
or
\[
x=\prod_{y\in V_x} (1-y)
\]
The real definition, which is more general, should be 
\[
yRx\mbox{ iff when we substitute } y=1\mbox{ in $\Bh_x$, we get that }\Bh_x=0.
\]
This definition is good in a more general context.

Suppose $y_1,y_2$ attack $x$ {\em jointly}. This means that $x=0$ only if both $y_1=y_2=1$. See \cite{459-9} for a discussion of joint attacks.

The equation for that is 
\[
x =1-y_1y_2
\]
Given a general equation 
\[
x =\Bh_x(V_x)
\]
for example 
\[
x=\Bh_x(y_1,y_2,z) =(1-z)(1-y_1,y_2)
\]
We define the notion for $V^0_x\subseteq V_x$ of joint attack as follows.

$V^0_x$ attack $x$ jointly if the substitution of $u=1$ for all variables in $V^0_x$ makes $\Bh_x=0$ and for no proper subset of $V^0_x$ do we have this property.

So in the above example, $y_1,y_2$ attack $x$ jointly and $z$ attacks $x$ singly.
}
\end{enumerate}
\item Let $\CB=(S,R)$ be a network and let $\BBE=Eq(\CB)$ be its system of equations.

$\BBE$ is a system of equations as in (a) above.  Let $B\subseteq V$ be some of the variables in $V$.  Let \Bf\ be a function giving numerical values 0 to the variables in $B$.

Let $\BBE_{\Bf}$ be the system of equations obtained from $\BBE$ by substituting the values $\Bf(u)$ in the equations for the variables of $B$.  The variables of $\BBE_{\Bf}$ are $V-B$. Consider now the argumentation network
\[
\CA_{\BBE_{\Bf}} =(S_{\BBE_{\Bf}}, R_{\BBE_{\Bf}}).
\]
We say that $\CA_{\BBE_{\Bf}}$ was derived from $\CB$ using \Bf.  We can also use the notation $\CB_{\Bf}$ or $\CB_B$.
\end{enumerate}
\end{definition}

\begin{example}\label{EFeb07-4}
Let us use the network of Figure \ref{FFeb07-5} to illustrate the process outlined in Remark \ref{RFeb06-1}.  This figure is used extensively in \cite{459-2} and also quoted in \cite{459-3}.

\begin{figure}
\centering
\setlength{\unitlength}{0.00083333in}
\begingroup\makeatletter\ifx\SetFigFont\undefined%
\gdef\SetFigFont#1#2#3#4#5{%
  \reset@font\fontsize{#1}{#2pt}%
  \fontfamily{#3}\fontseries{#4}\fontshape{#5}%
  \selectfont}%
\fi\endgroup%
{\renewcommand{\dashlinestretch}{30}
\begin{picture}(3417,1480)(0,-10)
\put(3402,1025){\makebox(0,0)[b]{\smash{{\SetFigFont{10}{12.0}{\rmdefault}{\mddefault}{\updefault}$h$}}}}
\path(117,1255)(642,805)
\blacken\path(531.365,860.317)(642.000,805.000)(570.413,905.873)(531.365,860.317)
\path(592,615)(117,205)
\blacken\path(188.238,306.120)(117.000,205.000)(227.443,260.699)(188.238,306.120)
\path(642,805)(1242,1255)
\blacken\path(1164.000,1159.000)(1242.000,1255.000)(1128.000,1207.000)(1164.000,1159.000)
\path(1242,1255)(2142,205)
\blacken\path(2041.127,276.587)(2142.000,205.000)(2086.683,315.635)(2041.127,276.587)
\path(692,610)(1242,205)
\blacken\path(1127.583,251.997)(1242.000,205.000)(1163.160,300.311)(1127.583,251.997)
\path(1392,130)(1992,130)
\blacken\path(1872.000,100.000)(1992.000,130.000)(1872.000,160.000)(1872.000,100.000)
\path(1992,55)(1392,55)
\blacken\path(1512.000,85.000)(1392.000,55.000)(1512.000,25.000)(1512.000,85.000)
\path(2742,430)(2292,130)
\blacken\path(2375.205,221.526)(2292.000,130.000)(2408.487,171.603)(2375.205,221.526)
\path(3362,985)(2927,610)
\blacken\path(2998.301,711.075)(2927.000,610.000)(3037.477,665.630)(2998.301,711.075)
\path(2442,1255)(3257,1105)
\blacken\path(3133.552,1097.217)(3257.000,1105.000)(3144.412,1156.226)(3133.552,1097.217)
\path(2142,205)(2367,1255)
\blacken\path(2371.191,1131.378)(2367.000,1255.000)(2312.522,1143.950)(2371.191,1131.378)
\put(42,1330){\makebox(0,0)[b]{\smash{{\SetFigFont{10}{12.0}{\rmdefault}{\mddefault}{\updefault}$a$}}}}
\put(42,55){\makebox(0,0)[b]{\smash{{\SetFigFont{10}{12.0}{\rmdefault}{\mddefault}{\updefault}$c$}}}}
\put(642,655){\makebox(0,0)[b]{\smash{{\SetFigFont{10}{12.0}{\rmdefault}{\mddefault}{\updefault}$b$}}}}
\put(1242,1330){\makebox(0,0)[b]{\smash{{\SetFigFont{10}{12.0}{\rmdefault}{\mddefault}{\updefault}$d$}}}}
\put(1242,55){\makebox(0,0)[b]{\smash{{\SetFigFont{10}{12.0}{\rmdefault}{\mddefault}{\updefault}$e$}}}}
\put(2142,55){\makebox(0,0)[b]{\smash{{\SetFigFont{10}{12.0}{\rmdefault}{\mddefault}{\updefault}$f$}}}}
\put(2442,1330){\makebox(0,0)[b]{\smash{{\SetFigFont{10}{12.0}{\rmdefault}{\mddefault}{\updefault}$g$}}}}
\put(2832,460){\makebox(0,0)[b]{\smash{{\SetFigFont{10}{12.0}{\rmdefault}{\mddefault}{\updefault}$i$}}}}
\path(42,205)(42,1255)
\blacken\path(72.000,1135.000)(42.000,1255.000)(12.000,1135.000)(72.000,1135.000)
\end{picture}
}
\caption{}\label{FFeb07-5}
\end{figure}

The variables of this figure are 
\[
V=\{a,b,c,d,e,f,g,h,i\}
\]
The equations are, using $Eq_{\rm inverse}$ as follows:

\begin{enumerate}
\item $a =1-c$
\item $b=1-a$
\item $c=1-b$
\item $d=1-b$
\item $e=(1-b)(1-f)$
\item $f=(1-e)(1-d)(1-i)$
\item $g=(1-f)$
\item $h=(1-g)$
\item $i=(1-h)$ 
\end{enumerate}
Let us take $B =\{a\}$ and let \Bf\ be the function making $a=0$ (i.e.\ $\Bf(a)=0$).  (This is a loop-busting move, breaking the loop $\{a,b,c\}$).

The new equations for $a$ are
\begin{itemlist}{xxxx}
\item [$1^*a$.] $a =(1-c)Z(a)$
\item [$1^*b$.] $Z(a) =0$
\end{itemlist}
or we can simply write 
\begin{itemlist}{xxxx}
\item [$1^*$.] $a=0$
\end{itemlist}
Substituting this value in the equations and solving we get the new system of equations for the unknown variable as follows
\begin{enumerate}
\item $b=1$, known value
\item $c=0$, known value
\item $d=0$, known value
\item $f=1-i$
\item $g=1-f$
\item $h=1-g$
\item $i=1-h$
\end{enumerate}
We get the solution function $\Bf_1$ giving the known values to the variables $a=0, b=1, c=0, d=0, e=0$ (these are the variables of rank 2) and the new system of equations (6), (7), (8), (9).  Using item (1) of Definition \ref{DFeb07-3}, we get the derived network in Figure \ref{FFeb07-6}.

\begin{figure}
\centering
\setlength{\unitlength}{0.00083333in}
\begingroup\makeatletter\ifx\SetFigFont\undefined%
\gdef\SetFigFont#1#2#3#4#5{%
  \reset@font\fontsize{#1}{#2pt}%
  \fontfamily{#3}\fontseries{#4}\fontshape{#5}%
  \selectfont}%
\fi\endgroup%
{\renewcommand{\dashlinestretch}{30}
\begin{picture}(1557,1405)(0,-10)
\put(15,665){\makebox(0,0)[b]{\smash{{\SetFigFont{10}{12.0}{\rmdefault}{\mddefault}{\updefault}$f$}}}}
\path(940,1300)(1540,840)
\blacken\path(1426.514,889.204)(1540.000,840.000)(1463.020,936.820)(1426.514,889.204)
\path(1545,570)(955,100)
\blacken\path(1030.167,198.234)(955.000,100.000)(1067.552,151.304)(1030.167,198.234)
\path(600,125)(30,590)
\blacken\path(141.948,537.391)(30.000,590.000)(104.020,490.899)(141.948,537.391)
\put(790,1255){\makebox(0,0)[b]{\smash{{\SetFigFont{10}{12.0}{\rmdefault}{\mddefault}{\updefault}$g$}}}}
\put(800,55){\makebox(0,0)[b]{\smash{{\SetFigFont{10}{12.0}{\rmdefault}{\mddefault}{\updefault}$i$}}}}
\put(1625,660){\makebox(0,0)[b]{\smash{{\SetFigFont{10}{12.0}{\rmdefault}{\mddefault}{\updefault}$h$}}}}
\path(35,850)(640,1315)
\blacken\path(563.138,1218.087)(640.000,1315.000)(526.574,1265.659)(563.138,1218.087)
\end{picture}
}
\caption{}\label{FFeb07-6}
\end{figure}

We can continue now with this loop and choose a loop-busting variable say $B'=\{i\}$. We substitute $i=0$ in the equations and get $f=1, g=0, h=1, i=0$ (these are the variables of rank 1).  We extend the function $\Bf_1$ to be $\Bf_2$ giving these values.

We thus get the extension $\{b, f, h\}$ (these are the variables which get value 1 from $\Bf_2$).  We also get clear ranks for the variables for the particular protocol  $\BBP = ( B,B' ) =( \{a\}, \{i\} )$.
\end{example}

\begin{example}\label{EFeb08-1}
Let us do another example. We use Figure \ref{FF1} item b.  We note that in Example \ref{FE2}, item (e), we make $\phi =0$.  This means we start with $B_1=\{\phi\}$.

We manipulated the equations in item (e) of Example \ref{FE2} and got the remaining equation (1)--(4), namely
\begin{enumerate}
\item $\alpha =1-\gamma$
\item $\delta =1-\alpha$
\item $\beta =1-\delta$
\item $\gamma =1-\beta$
\end{enumerate}
We proceeded in item (e) of Example \ref{FE2} to find two solutions to these equations. However, if we follow the procedures of Remark \ref{RFeb06-1}, we need now to extract a new network out of equations (1)--(4), and keep in mind the partial function $\Bf_1(\phi)=0$.  The new network is presented in Figure \ref{FFeb08-2}.

\begin{figure}
\centering
\setlength{\unitlength}{0.00083333in}
\begingroup\makeatletter\ifx\SetFigFont\undefined%
\gdef\SetFigFont#1#2#3#4#5{%
  \reset@font\fontsize{#1}{#2pt}%
  \fontfamily{#3}\fontseries{#4}\fontshape{#5}%
  \selectfont}%
\fi\endgroup%
{\renewcommand{\dashlinestretch}{30}
\begin{picture}(1557,1405)(0,-10)
\put(15,665){\makebox(0,0)[b]{\smash{{\SetFigFont{10}{12.0}{\rmdefault}{\mddefault}{\updefault}$\delta$}}}}
\path(940,1300)(1540,840)
\blacken\path(1426.514,889.204)(1540.000,840.000)(1463.020,936.820)(1426.514,889.204)
\path(1545,570)(955,100)
\blacken\path(1030.167,198.234)(955.000,100.000)(1067.552,151.304)(1030.167,198.234)
\path(600,125)(30,590)
\blacken\path(141.948,537.391)(30.000,590.000)(104.020,490.899)(141.948,537.391)
\put(790,1255){\makebox(0,0)[b]{\smash{{\SetFigFont{10}{12.0}{\rmdefault}{\mddefault}{\updefault}$\beta$}}}}
\put(800,55){\makebox(0,0)[b]{\smash{{\SetFigFont{10}{12.0}{\rmdefault}{\mddefault}{\updefault}$\alpha$}}}}
\put(1525,660){\makebox(0,0)[b]{\smash{{\SetFigFont{10}{12.0}{\rmdefault}{\mddefault}{\updefault}$\gamma$}}}}
\path(35,850)(640,1315)
\blacken\path(563.138,1218.087)(640.000,1315.000)(526.574,1265.659)(563.138,1218.087)
\end{picture}
}
\caption{}\label{FFeb08-2}
\end{figure}

We can now proceed by choosing a new loop-busting set $B_2$. We have four options here. Choosing $B_2$ to be $\alpha$ or $\beta$ will make the extension $\alpha=\beta=0, \delta=\gamma=1$ and choosing $B_2$ to the $\gamma$ or $\delta$ will give the extension $\alpha =\beta=1, \gamma=\delta=0$.

These are also the solutions we got in item (e) of Definition \ref{FE2}.
\end{example}

We now want to define procedures which will give us the CF2 extensions of Baroni {\em et al.}.  We need to give a protocol of which $B$ to use, as outlined in Definition \ref{DFeb07-3}.

\begin{definition}\label{DFeb07-7}
\begin{enumerate}
\item Let $\CA=(S,R)$ be an argumentation network.  Define $x\approx y$ iff $x=y$ or for some loop $u_1Ru_2, u_2Ru_3\comma u_nRu_1$ we have $x,y\in \{ u_1\comma u_n\}$. This is an equivalence relation on $S$. Let $S^*$ be a set of equivalence classes.\footnote{These are the maximal strongly connected sets in the terminology of \cite{459-1}.} Define $R^*$ on $S^*$ by $x^*R^*y^*$ iff for some $x'\approx x, y'\approx y$ we have $x'Ry'$.

$(S^*, R^*)$ is an ordering without loops.
\item Let $S_x\subseteq S$ be an $\approx$ equivalence class of $x\in S$. We say that $S_x$ is top-class if the following holds
\begin{itemize}
\item $yRz\wedge z\approx x\to y\approx x$. 
\end{itemize}
The above means that no other disjoint class $S_y$ attacks any member of $S_x$.

Let $C\subseteq S_x$ be a maximal subset of conflict free points in $S_x$. Let $B =S_x-C$. Then using $B$ in the protocol of Remark \ref{RFeb06-1} shall be referred to as {\em using the CF2 protocol}.
\end{enumerate}
\end{definition}

\begin{lemma}\label{LFeb07-8}
In the notation of Definition \ref{DFeb07-7}, if we apply the protocol of Remark \ref{RFeb06-1}, we get that all elements of $B$ solve to 0 and all elements of $C$ solve to 1.
\end{lemma}

\begin{proof}
Let $x$ be an element of $C$. The equation for $x$ is 
\[
x =\Bh_x(V_x)
\]
where $V_x=\{y|y \mbox{ attacks } x\}$.

$\Bh_x$ can be the expression $1-\max \{y|y\in V_x\}$ according to $Eq_{\max}$ or $\prod_{y\in V_x}(1-y)$ for resp. $Eq_{\rm inverse}$ or any other choice as long as the following condition holds.

\begin{itemize}
\item If for all $y \in V_x, y=0$ then $\Bh_x(V_x)=1$.
\end{itemize}

The important point to note is that since $S_x$ is a top loop, all attackers of $x$ are in $S_x$, and since $x\in C$ and $C$ is a {\em maximal} conflict free set, all attackers of $x$ are not in $C$, so they are in $B$ and so they are 0.  Hence $x=1$.\footnote{Obviously our proof does not work unless $C$ is maximal conflict free set.  So the LB semantics can yield the CF2 semantics only because Baroni {\em et al.} chose to take {\em maximal} conflict free set.  I shall ask Baroni for his reasons for this choice.}

The above consideration shows that $B$ can represent $C$, and so Baroni {\em et al.} conflict free choice $C$ can be represented as our equational loop-busting set $B$.  We now continue our protocol with $B$ and get a new network $\CA_B$ (note we are using the notation of Definition \ref{DFeb07-7} and our original network was $\CA =(S,R)$.)
\end{proof}

\begin{remark}\label{RFeb07-9}
Let $\CA=(S,R)$ be a network and suppose it does contain points that are not attacked. If we choose $B=\varnothing$ and apply our procedure of Remark \ref{RFeb06-1}, what do we get? 

The procedure solves the equations as much as possible to get a new network $\CA_1$ and a function $\Bf_1$ giving numerical values nodes which are in or out in the grounded extension. In other words, they have a $\{0,1\}$ numerical value. Thus what we get is the rest of the network after the grounded extension has been eliminated.
\end{remark}

\begin{example}\label{EFeb07-10}
Consider the network of Figure \ref{FFeb07-11}. Begin the procedure with $B_1=\varnothing$. We get $\Bf_1(a)=1, \Bf_1(b)=0$ and the remaining network $\CA_1$ of Figure \ref{FFeb07-12}.

\begin{figure}
\centering
\setlength{\unitlength}{0.00083333in}
\begingroup\makeatletter\ifx\SetFigFont\undefined%
\gdef\SetFigFont#1#2#3#4#5{%
  \reset@font\fontsize{#1}{#2pt}%
  \fontfamily{#3}\fontseries{#4}\fontshape{#5}%
  \selectfont}%
\fi\endgroup%
{\renewcommand{\dashlinestretch}{30}
\begin{picture}(1007,1383)(0,-10)
\put(790,655){\makebox(0,0)[b]{\smash{{\SetFigFont{10}{12.0}{\rmdefault}{\mddefault}{\updefault}$d$}}}}
\path(320,605)(560,200)
\blacken\path(473.015,287.941)(560.000,200.000)(524.632,318.529)(473.015,287.941)
\path(780,610)(600,195)
\blacken\path(620.227,317.028)(600.000,195.000)(675.273,293.153)(620.227,317.028)
\path(765,825)(765,828)(764,834)
	(763,844)(762,856)(761,872)
	(760,888)(760,905)(760,924)
	(762,943)(765,964)(770,985)
	(776,1004)(782,1020)(787,1030)
	(790,1037)(791,1042)(793,1045)
	(794,1047)(797,1050)(801,1054)
	(809,1060)(820,1068)(835,1075)
	(851,1080)(866,1082)(878,1084)
	(887,1084)(894,1083)(900,1082)
	(906,1081)(913,1080)(922,1077)
	(933,1073)(946,1068)(960,1060)
	(974,1049)(983,1039)(988,1031)
	(991,1025)(992,1021)(993,1017)
	(993,1011)(994,1001)(995,987)
	(995,970)(993,956)(990,943)
	(987,935)(986,929)(984,925)
	(983,923)(981,921)(979,917)
	(975,912)(968,904)(958,893)
	(945,880)(927,865)(910,853)
	(893,844)(877,836)(862,829)
	(848,823)(825,815)
\blacken\path(928.484,882.757)(825.000,815.000)(948.195,826.088)(928.484,882.757)
\put(295,645){\makebox(0,0)[b]{\smash{{\SetFigFont{10}{12.0}{\rmdefault}{\mddefault}{\updefault}$b$}}}}
\put(595,55){\makebox(0,0)[b]{\smash{{\SetFigFont{10}{12.0}{\rmdefault}{\mddefault}{\updefault}$c$}}}}
\put(15,1245){\makebox(0,0)[b]{\smash{{\SetFigFont{10}{12.0}{\rmdefault}{\mddefault}{\updefault}$a$}}}}
\path(30,1215)(260,805)
\blacken\path(175.126,894.980)(260.000,805.000)(227.454,924.335)(175.126,894.980)
\end{picture}
}
\caption{}\label{FFeb07-11}
\end{figure}

\begin{figure}
\centering
\setlength{\unitlength}{0.00083333in}
\begingroup\makeatletter\ifx\SetFigFont\undefined%
\gdef\SetFigFont#1#2#3#4#5{%
  \reset@font\fontsize{#1}{#2pt}%
  \fontfamily{#3}\fontseries{#4}\fontshape{#5}%
  \selectfont}%
\fi\endgroup%
{\renewcommand{\dashlinestretch}{30}
\begin{picture}(427,1111)(0,-10)
\path(200,610)(20,195)
\blacken\path(40.227,317.028)(20.000,195.000)(95.273,293.153)(40.227,317.028)
\path(185,825)(185,828)(184,834)
	(183,844)(182,856)(181,872)
	(180,888)(180,905)(180,924)
	(182,943)(185,964)(190,985)
	(196,1004)(202,1020)(207,1030)
	(210,1037)(211,1042)(213,1045)
	(214,1047)(217,1050)(221,1054)
	(229,1060)(240,1068)(255,1075)
	(271,1080)(286,1082)(298,1084)
	(307,1084)(314,1083)(320,1082)
	(326,1081)(333,1080)(342,1077)
	(353,1073)(366,1068)(380,1060)
	(394,1049)(403,1039)(408,1031)
	(411,1025)(412,1021)(413,1017)
	(413,1011)(414,1001)(415,987)
	(415,970)(413,956)(410,943)
	(407,935)(406,929)(404,925)
	(403,923)(401,921)(399,917)
	(395,912)(388,904)(378,893)
	(365,880)(347,865)(330,853)
	(313,844)(297,836)(282,829)
	(268,823)(245,815)
\blacken\path(348.484,882.757)(245.000,815.000)(368.195,826.088)(348.484,882.757)
\put(15,55){\makebox(0,0)[b]{\smash{{\SetFigFont{10}{12.0}{\rmdefault}{\mddefault}{\updefault}$c$}}}}
\put(210,655){\makebox(0,0)[b]{\smash{{\SetFigFont{10}{12.0}{\rmdefault}{\mddefault}{\updefault}$d$}}}}
\end{picture}
}
\caption{}\label{FFeb07-12}
\end{figure}

We can now continue with the loop-busting set $B_1=\{d\}$, follow the procedure for $\CA_1$ and get $d=0, c=1$.  Thus the solution with the loop buster sets $B_1=\varnothing, B_2=\{d\}$ yields the function $\Bf_2(a)=1, \Bf_2(b)=0, \Bf_2(d)=0, \Bf_2(c)=1$.

The corresponding extension is $\{a,c\}$.
\end{example}

\begin{remark}\label{RFeb07-13}
\begin{enumerate}
\item We are now in a position to define our LB$\CM$ extensions. We see from the discussions so far that we can use the following procedure to find $\{0,1\}$ functions \Bf\ on an argumentation network $\CA_0=(S_0,R_0)$.

{\bf Step 1.}\\
Define $\CA^*_0 =( S^*_0, R^*_0 )$ as in Definition \ref{DFeb07-7}.  Choose a loop-busting subset $B_1$ for the top loops of $\CA^*_0$, following some meta-level considerations, i.e.\ satisfying $\CM$.

{\bf Step 2.}\\
Apply the procedure of Remark  \ref{RFeb06-1}, using $B_1$ and $\CA_0$ and get $\Bf_1$ and  $\CA_1$.

Recall that following the terminology of Remark \ref{RFeb06-1} item 3, all elements in $S_0 - S_1$ have rank 1. These are the elements instantiated to numerical values  at Step 1.
Furthermore $\Bf_1$ is a $\{0,1\}$ function. Also note that $\CA_1 = (S_1 , R_0 \cap S_1 )$. 

{\bf Step 3.}\\
Choose a new loop-busting set $B_2$ for the top loops of $\CA_1$ using our meta-level considerations $\CM$.

{\bf Step 4.}\\
Go to apply step 2 to $\CA_1$ using $B_2$ and obtain $\Bf_2$ and $\CA_2$.

Also identify the elements of $S_2 - S_1$  as the elements of rank 2.

{\bf Step $n+2$.}\\
Continue until you get $\CA_{n+3}=\varnothing$.

The function $\Bf_{n+2}$ will be total on $S_0$ and will give you the extension 
\[
E(B_1,B_2\comma B_{n+2})=\{x|\Bf_{n+2}(x)=1\}.
\]
All elements in the network have a clearly defined rank, it being the step in which they were instantiated to numerical value in $\{0,1\}$.

\item We define the semantics LB$\CM$ as the family of all the extensions of the form $E(B_1\comma B_{n+2})$, for all possible choices of $B_i$ allowable by $\CM$.
\end{enumerate}
\end{remark}

We now want to proceed to define our loop-busting semantics LB1, LB2, LB3, and LB4.  We want to use the notions of minimal absolute and minimal relative. We already remarked in item (2) of Remark \ref{FR1} that the above notions need to be adjusted. We now have the tools to do so.

\begin{definition}[Computational loop-busting set]\label{DFeb08-3}
\begin{enumerate}
\item Let $(S,R)$ be a complete loop, as defined in Definition \ref{FD1}.  Let $B_1$ be a subset of $S$. We say that $B_1$ is a computational loop-buster if, when we follow the procedure outlined in Remark \ref{RFeb06-1}, and get the function $\Bf_1$ as described there, we have that 
\[
B'_1 =\{x|\Bf_1 (x)=0\}
\]
is a loop-busting set (according to Definition \ref{FD1}), namely $B'_1$ intersects every loop cycle in $S$.

Note that what this definition says is very simple. There is no need  for $B_1$ itself to intersect every loop $C$. Making all points in $B_1$ equal 0 in the equations would make more points 0, namely we get $B'_1$ and $B'_1$ does intersect every loop.  If this is the case, we say that $B_1$ is a computational loop-buster.
\item A computational loop-buster is minimal absolute if there is no smaller computational loop buster with a smaller number of elements.  

$B_1$ is a minimal relative computational loop-buster, if no proper subset of it is such.
\end{enumerate}
\end{definition}

\begin{definition}\label{DFeb07-14}
We define the following loop-busting semantics.
\begin{enumerate}
\item LB1\\
We are allowed to choose only loop-busting sets which contain elements from top loops (as defined in Definition \ref{DFeb07-7}, item (1)), and which are computationally minimal absolute (as defined in Definition \ref{DFeb08-3}). 

So in other words, if our network is $(S,R)$ and $S_x, S_y,\ldots$ are all the top loops then our set $B$ is a subset of $S_x\cup S_y\cup\ldots$ and it is a computational minimal absolute set for any loop in any top loop.

\item LB2 (to be proved equivalent to the CF2 semantics)\\
We are allowed to choose only loop-busting sets obtained from maximal conflict free subsets of top loops, as defined in Definition \ref{DFeb07-7}, item (2).
\item LB3\\
We are allowed to choose only loop-busting sets containing elements from top loops (as defined in Definition \ref{DFeb07-7}, item (1)), and which are computational minimal relative (as defined in Definition \ref{DFeb08-3}).
\item LB4 (Directional Shkop semantics of \cite{459-3})\\
We are allowed to choose loop-busting sets containing elements from top loops (as defined in Definition \ref{DFeb07-7}, item (1)) and which are a single loop element.

So this semantics busts top loops one at a time.
\end{enumerate}
Note that the LB4 semantics gives rise to the same extensions as LB3. The reason is that if we start with a computationally minimal relative set $B$, we can substitute its elements one by one following the protocol of LB4 semantics and we never bust all loops until we substitute the last element , because the set is minimal relative. But now we are doing an equivalent LB4 semantics!
\end{definition}

We now show that our LB2 semantics is the same as CF2.

The following is a definition of CF2 extensions, as given in \cite[Definition 4]{459-2}.  The notion of strongly connected set used was defined in Definition \ref{DFeb07-7}.

\begin{definition}\label{DFeb11-1}
Let $\CA =(S,R)$ be an argumentation network and let $E\subseteq S$ be a set of arguments, then $E$ is a CF2 extension of $\CA$ iff
\begin{enumerate}
\item If $(S,R)$ itself is a strongly connected set then $E$ is a maximal conflict free subset of $S$.
\item Otherwise, for every $C$ where $C$ is a maximal strongly connected subset of $S$, we have that the set $C\cap E$ is a CF2 extension of the network $(T^C_1,R_1)$, where
\[
\begin{array}{l}
T^C_1=C-\{x|\exists y \in E((y,x)\in R\wedge y \not \in C\}\\
R_1 = R\cap (T^C_1 \times T^C_1).
\end{array}\]
\end{enumerate}
\end{definition}

\begin{theorem}[$LB2 = CF2$]\label{TFeb11-1}
The semantics CF2 is the same as the semantics LB2.
\end{theorem}

\begin{proof}
\begin{enumerate}
\item We start by showing that every LB2 extension $E$ of a network $\CA_0=(S_0,R_0)$ is also a CF2 extension as defined in Definition \ref{DFeb11-1}. To achieve this goal we need to follow closely how the extension $E$ was defined in LB2 for $\CA_0$.

The LB2 semantics was defined in Definition \ref{DFeb07-14} by using the protocols of Remark \ref{RFeb07-13}. The loop-busting sets involved in these protocols were defined in item 2 of Definition \ref{DFeb07-7}.

Let us list the way the LB2 extension $E$ of $\CA_0$ is defined.
\begin{enumerate}
\item $E$ is defined according to item 2 of Remark \ref{RFeb07-13}. The extension is obtained in the form $E = E(B_1\comma B_{n+2})$, where each $B_{i+1}$ is a loop-buster on the top loops of the network $\CA_i =(S_i,R_i)$.
\item The loop buster was chosen according to the protocol (CF2 protocol) of item 2 of Definition \ref{DFeb07-7}.
\item The elements of $S_i-S_{i+1}$ are of rank $i$, where $i$ is the step in which they got a numerical value in $\{0,1\}$ by the function $\Bf_i$.

The function $\Bf_{n+1}$ gives numerical values in $\{0,1\}$ to all the elements of $S_0$ and we have 
\[
E =\{x\in S_0|\Bf_{n+1} (x) =1\}.
\]
\end{enumerate}
We are now going to use the rank $i$ to show that $E$ is a CF2 extension according to Definition \ref{DFeb11-1}.  We need a bit more preparation.
\begin{enumerate}
\setcounter{enumii}{3}
\item Let $\CA^*_0=(S^*_0,R^*_0)$ be the ordering without loop derived from $(S_0, R_0)$ as in Definition \ref{DFeb07-7}.  The element classes of $S^*_0$ are all the maximal strongly connected subsets of $S_0$
\end{enumerate}
We continue the proof by induction on 
$n+2$, being the number of steps required to define $E$.
\begin{enumerate}
\setcounter{enumii}{4}
\item Case $n+2=2 (n=0)$.

In this case we have that $(S_0,R_0)$ itself is a strongly connected set. Then $E$ is obtained from $B_1$ which satisfies the CF2 condition as in Step 1 of Remark \ref{RFeb07-13}. Using the notation of Step 1, we have $E=\{x|\Bf_1(x) =1\}$.  In this case $E$ is also a CF2 extension.
\item Case $n>0$\\
Take any strongly connected maximal subset $C$ of $(S_0,R_0)$. Let $k+1$ be the step at which all elements of $C$ get a numerical value. There are two possibilities:
\begin{enumerate}
\item At step $k$ some maximal subset $C' \subseteq C$ does not yet have numerical values.  $C'$ is a top loop in $\CA_k$.  In this case $B_k$ is a loop buster for $C'$ and makes it get a numerical value in $\CA_{k+1}$.

\item $C'$ is not a top loop in $\CA_k$, in which case $B_k$ busts some other loops and in the process of obtaining $\CA_{k+1}$, $C'$ disappears as all these elements get a numerical value. In fact, in this case it is Step $k$ which gives all elements of $C$ a numerical value.
\end{enumerate}
Let us now look at the set $T^C_1$, as defined in item 2 of Definition \ref{DFeb11-1}.
\[
T^C_1= C-\{\exists y \in (E-C)((y,x)\in R)\}
\]
The set $T^C_1$ is comprised from two parts, $T^C_{1,k}$ and $T^C_{1,k+1}$. Part $T^C_{1,k}$ are all points $z \in T^C_1$ that get numerical value at step $k$ by $\Bf_k$ and the set $T^C_{1,k+1}$ is the set of all points that get numerical values at step $k+1$ by the function $\Bf_{k+1}$.

In case (i) above, $T^C_{1,k+1}$ is still the loop $C'$, but still $T^C_{1,k}$ may be $\neq \varnothing$.

In case (ii), $T^C_{1,k+1} =\varnothing$. We ask is $E\cap T^C_1$ a CF2 extension of $T^C_1$ according to Definition \ref{DFeb11-1}?  The answer is yes.  The part $T^C_{1,k}$ is calculated traditionally and if there is a loop $C' = T^C_{1,k+1}$, it will be busted by $B_k$ which was chosen in LB2 to yield a maximal conflict free set.

We thus see through considerations (a)--(e) that LB2 $\subseteq$ CF2.
\end{enumerate}
The reader should see Remark \ref{RFeb13-1}, to appreciate the difference between the way LB2 and CF2 calculate their extensions.
\item We now prove the other direction, namely that every CF2 extension is also an LB2 extension.

Let $E_0$ be a CF2 extension of the network $\CA_0 =(S_0, R_0)$. We would like to define a sequence of LB2 loop-busters $B_1, B_2\comma B_{n+2}$ such that $E_0 =E(B_1\comma B_{n+2})$.  We choose $B_i$ by looking at $E_0$.

{\bf Step 1.}\\
Look at the top loops of $(S_0, R_0)$. Use $E_0$ to choose the loop-busters.

Let us look at top strongly connected sets of $(S_0, R_0)$. These are either single unattacked points $x$ for which the LB2 equation is $x=1$ (in agreement with CF2) or a loop $C$, for which CF2 gives a choice of maximal conflict-free subset $E_0 \cap C$. We can now choose our LB loop-buster $B_1$ to be
\[
B_1= \bigcup_{\mbox{top loops $C$ of }\CA_0} (C-E_0.
\]
We now apply step 1 of the LB2 procedure and get $\CA_1=(S_1,R_1)$. Again consider the top loops of $\CA_1$.  Let 
\[
B_1=\bigcup_{\mbox{top loops $C$ of }\CA_1} (C-E_0)
\]
We carry on in this manner and get 
\[
E(B_1,B_2,\ldots).
\]
To show that $E_0=E$ is not difficult. This is done along the lines of the proof of (1) above.
\end{enumerate}
\end{proof}

\begin{remark}\label{RFeb13-1}
The way we compute the extensions of LB2 are not synchronised with the way CF2 itself works. This can be seen from the network of Figure \ref{FFeb07-5}.  In this figure the top loop is $\{a,b,c\}$. By choosing the maximal conflict free set $\{c\}$, we are led to the loop-buster $B =\{b\}$.

In step 1 we propagate the attacks (solve the equations) to get numerical values for as many variables as we can.

We get in LB2
\[
a=0, c=1, b=0, d=1, e=1, f=0, g=1, h=0 \mbox{ and } i=1.
\]
We get the extension in one step
\[
E =\{c,d,e,g,i\}=E(\{b\}).
\]
In comparison, when we follow the CF2 procedures, we look at two loops, the strongly connected sets, $\{a,b,c\}$ and $\{e,f,g,h,i\}$.  

Step 1 of the LB2 procedure corresponds to what the CF2 definition does, namely treat the loop $\{a,b,c\}$. This we do by choosing $c=1$ and calculating $a=0, b=0$ and $d=1$.

We now look at the loop $\{e,f,g,h,i\}$ and take from it the elements attacked from outside it.  In this case we take the element $f$ attacked by $d$.  Thus we are left with $S_1=\{e,g,h,i\}$ and $R_1$ being $\{(g,h), (h,i)\}$.  The CF2 extension for $(S_1,R_1)$ is $\{e,g,i\}$. So the extension we get finally is $E =\{c,d,e,g,i\}$.

The intuition of LB2 is to solve equations and get numerical values as much as possible. In the process we bust loops, by making loop variables equal 0.  Because we are dealing with equations, we can make loop variables equal 0 one by one.

CF2 in comparison, is different.

\begin{enumerate}
\item It concentrates more on loops
\item It treats loops by choosing maximal conflict free sets and makes them 1.
\item Since it is not dealing with equations, it must make 1 batches of variables (maximal conflict free sets) and cannot do them one by one.  So LB2 has the same problem. It must make 0 batches of variables all in one go.
\end{enumerate}
See \cite{459-2} for an algorithm for CF2. I think they do it inductively as we did in this example.

\end{remark}


\begin{remark}[Computational complexity]\label{RFeb07-15}
The computational complexity of the LB semantics is daunting. This is because of the 
requirement that the loop-busting sets in LB1 and LB3 be computationally minimal.  This means that we have to solve many equations before we can choose our sets. So I don't regard LB1 and LB3 as practical. In comparison, LB4 is easy, we just choose a loop element based on the geometry of the network and proceed recursively.  This is simple and easy.

Two factors are in our favour. One mathematical and one social.  

The mathematical one is that since we are dealing with equation solving, there are tools available for our use such as {\sc Mathematica} or {\sc Matlab} or {\sc Maple} or NSolve which can help.

We note that there is actually no need to go to equations at all, in view of the soundness results in the representation Theorem \ref{FT2}.

We can work completely with argumentation networks.

On the social side, our argumentation community is blessed with many talented young researchers such as Sarah Gaggl and Stefan Woltran (see \cite{459-2}) and Wolfgang Dvorak, who write good algorithms, and who (should they take an interest in the LB semantics) would certainly find a way around complexity difficulties.

Perhaps the comparison in Remark \ref{RFeb13-1} shows the way of how the algorithms of \cite{459-2}  can be adapted to the LB semantics. 
\end{remark}

\begin{remark}[Comparison of LB1 semantics with the CF2 semantics]\label{RFeb08-4}
Let us make a quick comparison.
\begin{enumerate}
\item LB1 give the correct exact extension for even loops. CF2 gives more extensions.
\item LB1 does bust odd loops but gives less extensions than CF2.
\item LB1 remains conceptually within the attack concept framework.  It has the backing of hundreds of years of experience in approximate/guess/simplification of equations. CF2 uses the slightly out of sync concept of maximal conflict free sets, and although LB2 gives it an equational flavour, it is forced and not natural (see next item 4).
\item Most importantly, and this has been pointed out to me by Martin Caminada, CF2 is not robust against conceptual extensions, such as joint attacks.  This is shown in the next remark, \ref{RFeb08-5}.

We shall also see that the LB semantics is robust.
\end{enumerate}
\end{remark}

\begin{remark}[CF2 and joint attacks]\label{RFeb08-5}
\begin{enumerate}
\item In my paper \cite{459-9} of Fibring argumentation frames, I introduced the notion of joint attack of say two arguments $a$ and $b$ on a third argument $c$.  This means that $c$  is in only when both $a$ and $b$ are out.  I used the notation of Figure \ref{FFeb08-6}.

\begin{figure}
\centering
\setlength{\unitlength}{0.00083333in}
\begingroup\makeatletter\ifx\SetFigFont\undefined%
\gdef\SetFigFont#1#2#3#4#5{%
  \reset@font\fontsize{#1}{#2pt}%
  \fontfamily{#3}\fontseries{#4}\fontshape{#5}%
  \selectfont}%
\fi\endgroup%
{\renewcommand{\dashlinestretch}{30}
\begin{picture}(1207,1528)(0,-10)
\path(587,800)(12,1345)
\path(602,795)(602,190)
\blacken\path(572.000,310.000)(602.000,190.000)(632.000,310.000)(572.000,310.000)
\path(607,790)(1182,1345)
\put(597,798){\ellipse{124}{124}}
\put(17,1380){\makebox(0,0)[b]{\smash{{\SetFigFont{10}{12.0}{\rmdefault}{\mddefault}{\updefault}$a$}}}}
\put(1192,1390){\makebox(0,0)[b]{\smash{{\SetFigFont{10}{12.0}{\rmdefault}{\mddefault}{\updefault}$b$}}}}
\put(607,55){\makebox(0,0)[b]{\smash{{\SetFigFont{10}{12.0}{\rmdefault}{\mddefault}{\updefault}$c$}}}}
\end{picture}
}
\caption{}\label{FFeb08-6}
\end{figure}
The equation for $c$ is 
\[
c =1-ab.
\]
I also showed in the paper how to interpret joint attacks within ordinary argumentation networks, using for each node $e$ the new auxiliary points $x(e)$ and $y(e)$. The joint attack of Figure \ref{FFeb08-6} can be represented faithfully by Figure \ref{FFeb08-7}.

\begin{figure}
\centering
\setlength{\unitlength}{0.00083333in}
\begingroup\makeatletter\ifx\SetFigFont\undefined%
\gdef\SetFigFont#1#2#3#4#5{%
  \reset@font\fontsize{#1}{#2pt}%
  \fontfamily{#3}\fontseries{#4}\fontshape{#5}%
  \selectfont}%
\fi\endgroup%
{\renewcommand{\dashlinestretch}{30}
\begin{picture}(1210,2763)(0,-10)
\put(1195,2615){\makebox(0,0)[b]{\smash{{\SetFigFont{10}{12.0}{\rmdefault}{\mddefault}{\updefault}$b$}}}}
\path(409,1719)(674,1169)
\blacken\path(594.886,1264.084)(674.000,1169.000)(648.939,1290.128)(594.886,1264.084)
\path(930,1725)(710,1170)
\blacken\path(726.331,1292.610)(710.000,1170.000)(782.109,1270.500)(726.331,1292.610)
\path(1179,2562)(959,2007)
\blacken\path(975.331,2129.610)(959.000,2007.000)(1031.109,2107.500)(975.331,2129.610)
\path(699,876)(701,260)
\blacken\path(670.611,379.902)(701.000,260.000)(730.610,380.097)(670.611,379.902)
\put(15,2625){\makebox(0,0)[b]{\smash{{\SetFigFont{10}{12.0}{\rmdefault}{\mddefault}{\updefault}$a$}}}}
\put(305,1825){\makebox(0,0)[b]{\smash{{\SetFigFont{10}{12.0}{\rmdefault}{\mddefault}{\updefault}$x(a)$}}}}
\put(740,945){\makebox(0,0)[b]{\smash{{\SetFigFont{10}{12.0}{\rmdefault}{\mddefault}{\updefault}$y(c)$}}}}
\put(940,1830){\makebox(0,0)[b]{\smash{{\SetFigFont{10}{12.0}{\rmdefault}{\mddefault}{\updefault}$x(b)$}}}}
\put(700,55){\makebox(0,0)[b]{\smash{{\SetFigFont{10}{12.0}{\rmdefault}{\mddefault}{\updefault}$c$}}}}
\path(50,2565)(315,2015)
\blacken\path(235.886,2110.084)(315.000,2015.000)(289.939,2136.128)(235.886,2110.084)
\end{picture}
}
\caption{}\label{FFeb08-7}
\end{figure}

It is important to note that if we calculate the equations of Figure \ref{FFeb08-7} we get for $Eq_{\rm inverse}$
\[\begin{array}{l}
x(a) =1-a\\
x(b)=1-b\\
y(c)=ab\\
c=1-ab
\end{array}\]

If we use $Eq_{\max}$ we get 
\[\begin{array}{l}
x(a)=1-a\\
x(b)=1-b\\
y(c)=1-\max(1-a,1-b)\\
c=\max(1-a,1-b)=1-\min(a,b)
\end{array}
\]
It is clear that in either case the equation for $c$ does not depend on the auxiliary points. Therefore any equational computation to get extensions will not be affected by the auxiliary points.
\item  Let us now take a loop involving joint attacks. Suppose we have 3 items, $a,b$ and $c$ and enough money to buy only two. Thus buying any two items attacks jointly the buying the third item. We get the loop of Figure \ref{FFeb08-8}.

\begin{figure}
\centering
\setlength{\unitlength}{0.00087489in}
\begingroup\makeatletter\ifx\SetFigFont\undefined%
\gdef\SetFigFont#1#2#3#4#5{%
  \reset@font\fontsize{#1}{#2pt}%
  \fontfamily{#3}\fontseries{#4}\fontshape{#5}%
  \selectfont}%
\fi\endgroup%
{\renewcommand{\dashlinestretch}{30}
\begin{picture}(1872,1813)(0,-10)
\put(552,1180){\ellipse{128}{128}}
\put(597,730){\ellipse{128}{128}}
\put(237,910){\ellipse{128}{128}}
\path(192,1630)(552,1180)(1722,1180)
\path(12,1585)(237,910)(1722,1090)
\blacken\path(1606.482,1045.778)(1722.000,1090.000)(1599.262,1105.342)(1606.482,1045.778)
\path(1722,1000)(597,730)(372,280)
\path(552,1180)(282,280)
\blacken\path(287.747,403.560)(282.000,280.000)(345.217,386.319)(287.747,403.560)
\path(597,730)(102,1630)
\blacken\path(186.117,1539.312)(102.000,1630.000)(133.544,1510.397)(186.117,1539.312)
\path(237,910)(192,280)
\put(57,1675){\makebox(0,0)[b]{\smash{{\SetFigFont{10}{12.0}{\rmdefault}{\mddefault}{\updefault}$a$}}}}
\put(1857,1045){\makebox(0,0)[b]{\smash{{\SetFigFont{10}{12.0}{\rmdefault}{\mddefault}{\updefault}$b$}}}}
\put(282,55){\makebox(0,0)[b]{\smash{{\SetFigFont{10}{12.0}{\rmdefault}{\mddefault}{\updefault}$c$}}}}
\end{picture}
}
\caption{}\label{FFeb08-8}
\end{figure}

The representation of this figure using the auxiliary point into an ordinary argumentation network is presented in Figure \ref{FFeb08-9}.

\begin{figure}
\centering
\setlength{\unitlength}{0.00087489in}
\begingroup\makeatletter\ifx\SetFigFont\undefined%
\gdef\SetFigFont#1#2#3#4#5{%
  \reset@font\fontsize{#1}{#2pt}%
  \fontfamily{#3}\fontseries{#4}\fontshape{#5}%
  \selectfont}%
\fi\endgroup%
{\renewcommand{\dashlinestretch}{30}
\begin{picture}(3636,1800)(0,-10)
\path(2175,1605)(60,435)
\blacken\path(150.482,519.338)(60.000,435.000)(179.526,466.836)(150.482,519.338)
\path(1455,390)(690,1560)
\blacken\path(780.779,1475.981)(690.000,1560.000)(730.561,1443.146)(780.779,1475.981)
\path(2175,1605)(2175,1065)
\blacken\path(2145.000,1185.000)(2175.000,1065.000)(2205.000,1185.000)(2145.000,1185.000)
\path(1455,390)(2175,885)
\blacken\path(2093.111,792.295)(2175.000,885.000)(2059.119,841.738)(2093.111,792.295)
\path(915,1695)(1320,1695)
\blacken\path(1200.000,1665.000)(1320.000,1695.000)(1200.000,1725.000)(1200.000,1665.000)
\path(1590,1695)(1995,1695)
\blacken\path(1875.000,1665.000)(1995.000,1695.000)(1875.000,1725.000)(1875.000,1665.000)
\path(825,300)(1230,300)
\blacken\path(1110.000,270.000)(1230.000,300.000)(1110.000,330.000)(1110.000,270.000)
\path(240,300)(645,300)
\blacken\path(525.000,270.000)(645.000,300.000)(525.000,330.000)(525.000,270.000)
\path(2400,975)(2805,975)
\blacken\path(2685.000,945.000)(2805.000,975.000)(2685.000,1005.000)(2685.000,945.000)
\path(3570,1065)(735,1560)
\blacken\path(858.372,1568.913)(735.000,1560.000)(848.052,1509.807)(858.372,1568.913)
\path(3030,975)(3435,975)
\blacken\path(3315.000,945.000)(3435.000,975.000)(3315.000,1005.000)(3315.000,945.000)
\path(3615,885)(3615,884)(3616,883)
	(3616,881)(3617,878)(3618,873)
	(3619,867)(3621,859)(3622,850)
	(3623,839)(3624,827)(3624,814)
	(3624,799)(3622,784)(3619,768)
	(3615,752)(3609,734)(3601,716)
	(3591,698)(3578,678)(3562,659)
	(3542,638)(3519,616)(3492,594)
	(3459,570)(3421,545)(3378,519)
	(3328,492)(3272,463)(3210,435)
	(3158,413)(3104,391)(3050,370)
	(2998,350)(2947,331)(2899,313)
	(2854,297)(2813,282)(2775,269)
	(2741,256)(2710,245)(2682,235)
	(2657,226)(2634,218)(2613,210)
	(2594,203)(2576,196)(2559,190)
	(2542,184)(2525,178)(2508,172)
	(2489,166)(2468,160)(2445,153)
	(2419,146)(2390,139)(2357,132)
	(2319,124)(2277,115)(2231,106)
	(2178,96)(2121,87)(2058,76)
	(1990,66)(1918,56)(1841,47)
	(1761,38)(1680,30)(1598,24)
	(1518,19)(1439,15)(1363,13)
	(1290,12)(1219,12)(1151,14)
	(1086,16)(1024,19)(964,23)
	(906,27)(850,32)(796,38)
	(743,44)(692,51)(643,58)
	(595,65)(548,73)(503,81)
	(459,89)(418,96)(378,104)
	(340,112)(304,119)(271,127)
	(241,133)(214,139)(190,145)
	(169,150)(152,154)(137,157)
	(126,160)(118,162)(105,165)
\blacken\path(228.673,167.249)(105.000,165.000)(215.181,108.785)(228.673,167.249)
\put(645,1650){\makebox(0,0)[b]{\smash{{\SetFigFont{10}{12.0}{\rmdefault}{\mddefault}{\updefault}$y(a)$}}}}
\put(2175,1650){\makebox(0,0)[b]{\smash{{\SetFigFont{10}{12.0}{\rmdefault}{\mddefault}{\updefault}$x(a)$}}}}
\put(1455,255){\makebox(0,0)[b]{\smash{{\SetFigFont{10}{12.0}{\rmdefault}{\mddefault}{\updefault}$x(b)$}}}}
\put(15,255){\makebox(0,0)[b]{\smash{{\SetFigFont{10}{12.0}{\rmdefault}{\mddefault}{\updefault}$y(b)$}}}}
\put(2175,930){\makebox(0,0)[b]{\smash{{\SetFigFont{10}{12.0}{\rmdefault}{\mddefault}{\updefault}$y(c)$}}}}
\put(1455,1650){\makebox(0,0)[b]{\smash{{\SetFigFont{10}{12.0}{\rmdefault}{\mddefault}{\updefault}$a$}}}}
\put(735,255){\makebox(0,0)[b]{\smash{{\SetFigFont{10}{12.0}{\rmdefault}{\mddefault}{\updefault}$b$}}}}
\put(2895,930){\makebox(0,0)[b]{\smash{{\SetFigFont{10}{12.0}{\rmdefault}{\mddefault}{\updefault}$c$}}}}
\put(3615,930){\makebox(0,0)[b]{\smash{{\SetFigFont{10}{12.0}{\rmdefault}{\mddefault}{\updefault}$x(c)$}}}}
\end{picture}
}
\caption{}\label{FFeb08-9}
\end{figure}
In this figure, the set $\{a,b,c\}$ is conflict free in the expanded network with the auxiliary points but it is not so in the original network, contrary to intuition. So CF2 messes up the concept of joint attack.

In comparison, the LB semantics is not affected by the auxiliary points as we have already seen from the equations. Not being affected means that if we start with a network $(S,R)$ and move for whatever reason  to an expanded network $(S', R')$, containing additional  auxiliary points , then any extension $E'$ obtained traditionally for $(S',R')$ will endow a correct and acceptable extension  $E= E' \cap S$ of $(S,R)$, and furthermore all such correct extensions $E$ are so obtained.

 Let us find a loop-buster for Figure \ref{FFeb08-9}.  Note that this system is similar to Figure \ref{FF7}, being a 9-point loop.

A computational loop-buster would be, for example, $\{y(a), y(b)\}$.

This gives $a=b=1$ and $c=0$. $\{y(a), y(b), y(c)\}$ is not minimal so we cannot get the extension $\{a,b,c\}$.  

It is clear that our LB machinery works correctly here.
\end{enumerate}
\end{remark}

\section {Discussion and conclusion}
Let us summarise the progression of logical and conceptual steps involved in this paper:
 \begin{enumerate}
\item  We saw that the equational approach allows us to  associate in a one to one manner a  system of  equations with each argumentation network.
\item Finding a solution to the equations corresponds to finding an extension for the argumentation network.
\item  When a system of equation is difficult to handle, there is a well known and much used methodology in the equational area of perturbing the equations to make it them more manageable, in a way compatible with one's application area.
    \item  In the argumentation area odd and even loops are a bit of a problem
    \item  The counterpart of loops in the equational area  are cycles of variable dependencies, a well understood phenomena, traditionally handled by iterative solutions.
    \item We therefore suggest, inspired by equational thinking, the concept of loop-busting sets of variables which we perturb to be 0, and use them to modify the equations.
    \item This gives rise to the LB semantics for argumentation, which contains the CF2 semantics as  a special option, which we called LB2 (= CF2).

Note that LB2 is based on making arguments 0 , while CF2 is based on making arguments 1. So the agreement is non-trivial.
 \item We considered other options such as LB1, LB3 and LB4. We showed that they agree with traditional semantics on even loops and still repair odd loops.
    \item We compared the other LB semantics with CF2 and found them robust as far as conceptual changes such as joint attacks.
   \item The LB1, LB3 and LB4 semantics can be done by loop busting using repeatedly and recursively one element at a time. The LB2 semantics (which simulates CF2 and uses maximal conflict free sets) requires the use of all points of the loop busting set to be made 0 simulataneously. 

\end{enumerate}

We now say a few comments, comparing  this paper with \cite{459-3}. The Shkop semantics introduced in \cite{459-3}, corresponds to LB4. In \cite{459-3} argumentation loops are considered as created over time, where the arguments come as a result of agents creating situations by their temporal actions. This means that when an odd loop is created, we can identify one of the members of the loop as the temporally last argument which came into existence by some action and created the loop. The Shkop principle says that in this case, this argument is rejected (in the real world the action giving rise to it is annulled). This is mathematically equivalent to using this last argument as a loop buster and making it equal 0.

So the Shkop principle and the Shkop semantics of \cite{459-3} corresponds to LB4 semantics where the loop-busters are chosen using temporal information.

\section*{Acknowledgements}
I am grateful to Martin Caminada, Sarah Gaggl, and Stefan Woltran for helpful discussion.

Research done under ISF project: Integrating Logic and Network Reasoning.

\end{document}